\documentclass{article}
 \usepackage[preprint]{neurips_2026}
\usepackage[utf8]{inputenc} 
\usepackage[T1]{fontenc}    
\usepackage{hyperref}       
\usepackage{url}            
\usepackage{booktabs}       
\usepackage{amsfonts}       
\usepackage{nicefrac}       
\usepackage{microtype}      
\usepackage{xcolor}         
\usepackage{graphicx}       

\usepackage{amsthm}
\usepackage{amsmath,mathtools}
\usepackage{amssymb}

\usepackage{tabularx,array}

\usepackage{algorithm}
\usepackage{algorithmic}

\usepackage{caption}
\usepackage{newfloat}
\usepackage{listings}
\DeclareCaptionStyle{ruled}{labelfont=normalfont,labelsep=colon,strut=off}
\floatstyle{ruled}
\newfloat{listing}{tb}{lst}{}
\floatname{listing}{Listing}

\newtheorem{theorem}{Theorem}
\newtheorem{lemma}{Lemma}
\newtheorem{proposition}{Proposition}
\newtheorem{corollary}{Corollary}

\newtheorem{remark}{Remark}

\newcommand{\R}{\mathbb{R}}
\newcommand{\E}{\mathbb{E}}
\newcommand{\N}{\mathcal{N}}
\newcommand{\KL}{\mathrm{KL}}
\newcommand{\Var}{\operatorname{Var}}
\newcommand{\Cov}{\operatorname{Cov}}
\newcommand{\diag}{\operatorname{diag}}
\newcommand{\dd}{\,\mathrm{d}}

\usepackage{enumitem}
\newcommand{\ran}{\operatorname{ran}}
\newcommand{\corr}{\operatorname{corr}}
\newcommand{\osc}{\operatorname{osc}}
\newcommand{\TV}{\operatorname{TV}}
\newcommand{\MMSE}{\mathrm{MMSE}}
\DeclareMathOperator*{\argmin}{arg\,min}

\newcommand{\dzf}[1]{\frac{\Delta z_{#1}}{z_{#1}}}
\newcommand{\Har}{\mathcal{H}}
\newcommand{\Hil}{\mathbb{H}}
\newcommand{\Sm}{\mathcal{S}}
\newcommand{\Wind}{\mathcal{W}_T}

\title{PRISM: Principled Reference Identification for Schr\"odinger Bridge Models}
\author{%
Forouzan Fallah,
Yezhou Yang \\
    Arizona State University}
\begin{document}
\maketitle

\begin{abstract}
Schr\"odinger bridge models restore a clean signal from a degraded observation by following the conditional bridges of a reference process, yet this reference is chosen heuristically, typically white noise with a hand-tuned schedule. We develop PRISM, a theory of bridge reference design. We characterize the time-varying Gaussian references that remain exactly tractable with per-mode schedules: precisely those whose instantaneous covariances commute. We then prove an invisibility principle: with the exact drift and unlimited solver steps, every admissible reference recovers the true posterior. The choice of reference therefore matters only under finite computational resources.
For a fixed step budget, we derive the finite-step objective in closed form and prove that every optimal noise spectrum is proportional to $P_k$, the spectrum of information destroyed by the sensor, with a mode-independent constant $x^{*}(T) = (2\ln T)^{-1/2}(1+o(1))$. The analysis shows that noise color and temporal scheduling are interchangeable, and regularization provably shifts the optimal reference toward white noise. 
Experiments in Gaussian settings confirm the predicted orderings and the closed-form loss floors.
On FFHQ, the distortion--perception trade-off and spectral localization transfer, but white noise outperforms the matched reference; a pre-registered study that changes the training regime refutes ridge whitening as the explanation. A $2{\times}2$ mechanism study then traces the inversion to the non-Gaussian per-mode statistics of real images. 
PRISM turns reference design from a hyperparameter sweep into a calculation in the Gaussian regime, and locates exactly where real images break it.
\end{abstract}

\section{Introduction}
\label{sec:intro}
Schr\"odinger bridge models \cite{liu20232,de2021diffusion,shi2023diffusion} learn to transport samples between two distributions along the conditional bridges of a \emph{reference process}. This reference determines the intermediate states used during training, including how noise is distributed across spatial frequencies. Yet it is usually chosen by hand, most often as white noise with a heuristic scalar schedule. This practice treats the reference as a minor implementation detail. Fig.~\ref{fig:teasor} shows that it is not: at every solver budget there is an interior optimal noise level, and its location follows a simple, predictable rule.
\begin{figure}[!t]
    \centering
    \includegraphics[width=.9\columnwidth]{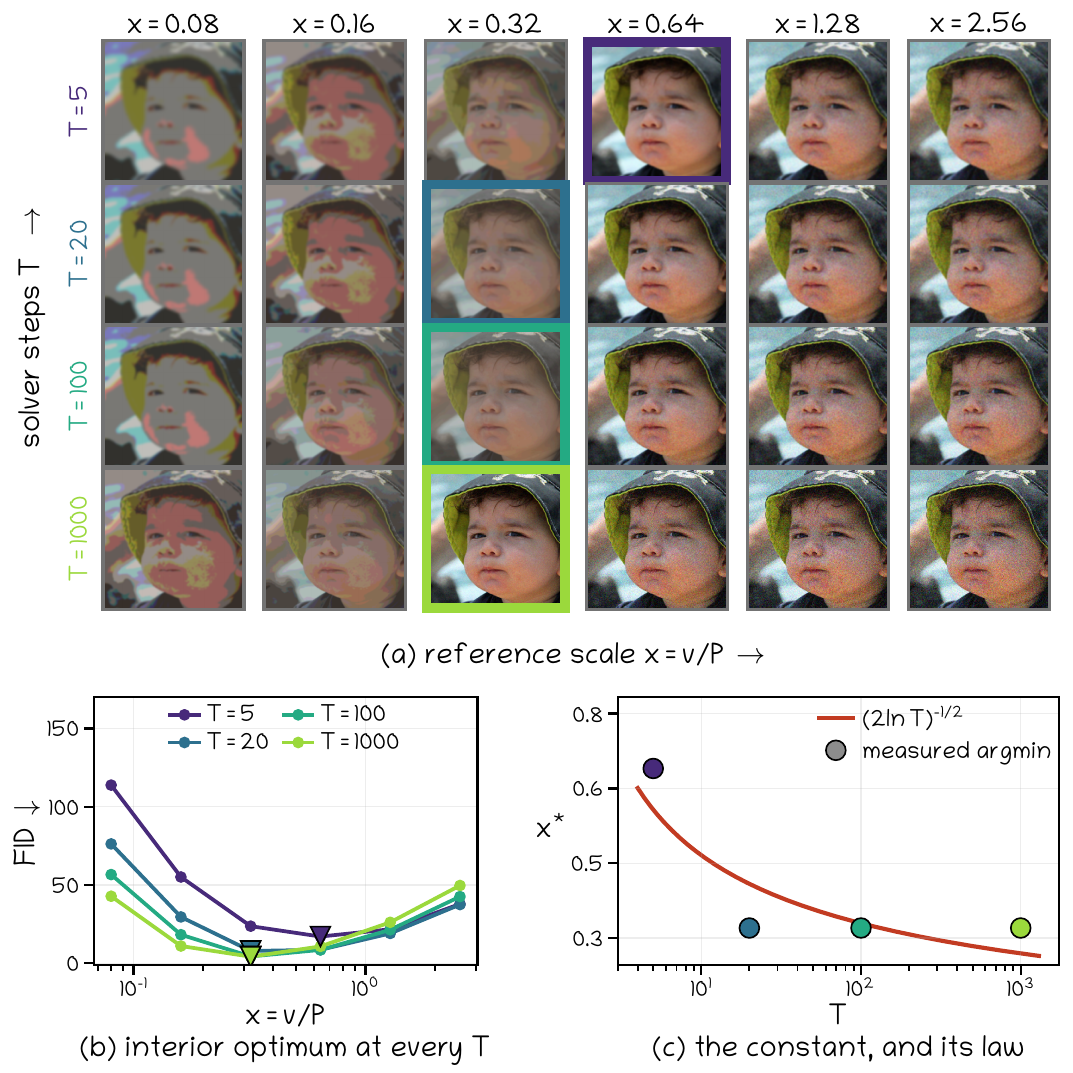}
    \caption{\textbf{There is a right amount of reference noise, and its size is predictable.} (a) Each row is a solver sampling budget $T$; each column shows the injected reference noise level, as a scale
    $x$ relative to the information destroyed by the sensor. Too little noise (left) fails to recover missing detail; too much (right) obscures the reconstruction. (b) Reconstruction error as a function of $x$, with one curve per $T$. Every curve has a minimum, showing that an intermediate noise level performs best. (c) The minimum's location follows Eq.~\eqref{eq:rule}, matching the measured optimum at every tested budget.}
    \label{fig:teasor}
\end{figure}
Diffusion schedules were hand-tuned until schedule theory \cite{karras2022elucidating,kingma2021variational} made their design principled. Recent work has shown empirically that the spectrum of injected noise matters in one-endpoint generative diffusion \cite{falck2025fourier,jiralerspong2025shaping,benita2026spectral,esteves2026spectrallyguided}.
That setting, however, has no measurement operator that identifies a preferred
spectrum. Color is therefore selected by search rather than derived, with no
theory describing when it should matter. This paper develops the analogous theory for bridge references, which we call PRISM (Principled Reference Identification for Schr\"odinger Bridge Models), and finds the answer more subtle than a single optimal formula: the reference is irrelevant in the ideal limit, while its optimal finite-resource design depends on solver steps or model error.

Our analysis proceeds in a solvable linear-Gaussian degradation model, where the observation is $x_1 = h_k x_0 + n$ per spatial frequency $k$, with transfer function $h_k$, noise spectrum $N_k$, and signal prior $S_k$. The central quantity is the Wiener residual spectrum
\begin{equation}
    P_k \;=\; {S_k N_k}/({h_k^2 S_k + N_k}),
    \label{eq:wiener-residual}
\end{equation}
which measures the posterior uncertainty remaining after observing $x_1$; where the sensor sees well, $P_k \approx N_k \approx 0$; where it is blind, $P_k \approx S_k$. Thus, $P_k$ is the spectrum of the destroyed information and is computable directly from the known degradation model.

\paragraph{Contributions.}
\begin{enumerate}
    \item \textbf{A tractability theorem (\S\ref{sec:tractability}).} We characterize the time-varying operator-valued references that admit $I^2$SB-style closed-form bridges and decouple into independent scalar modes: their instantaneous covariances must pairwise commute, equivalently simultaneously diagonalizable. We also give blockwise aliasing, linear-drift, and Hilbert-space extensions.
    \item \textbf{An invisibility principle (\S\ref{sec:invisibility}).} With the exact drift and $T\to\infty$ steps, the terminal law equals the true posterior for every reference variance $v>0$, and collapses to a point mass at $v=0$. The reference matters only once some resource (solver steps or model accuracy) is limited.
    \item \textbf{Finite-step theory (\S\ref{sec:finitestep}).} We derive the $T$-step terminal law in closed form. The per-mode KL depends on $(v,P)$ only through $x = v/P$, so every optimizer is proportional, $v_k = x^{*}P_k$. On uniform grids, $x^{*}(T) = (2\ln T)^{-1/2}(1+o(1))$ and the optimal KL is $\tfrac{\ln T}{2T^{2}}(1+o(1))$. A second change of variables, the $z$-identity, makes color and schedule exchangeable: under per-mode schedules discretization error is color-blind, and the optimal schedule is unique with a
    $4P/T$ floor. On clustered grids the color objective has several local minima, so the two cannot be designed independently.
    \item \textbf{Model error and learning (\S\ref{sec:modelerror}).} We extend the analysis to arbitrary (gain, bias) perturbations of the drift. Given the $z$-schedule and error profile, sampling is color-blind, and per-mode learning is color-blind for any scale-equivariant learner; color matters only where scale symmetry breaks. Shared ridge regularization breaks it in a definite direction, giving the sub-proportional law $v_k = x^{*}(nP_k)\cdot P_k$: less data whitens the optimal reference, which lands between white noise and $P_k$.
    \item \textbf{Experiments, and a refuted prediction (\S\ref{sec:experiments}).} Numerics and trained predictors confirm the predicted orderings and closed-form loss floors, and measure ridge whitening directly. On FFHQ \cite{ffhq} at $64{\times}64$ with a known degradation, the distortion--perception trade-off and the predicted spectral localization transfer, but white stays the better perceptual reference. A pre-registered two-regime study (\S\ref{sec:exp-regime}) refutes ridge whitening as the explanation, and a $2{\times}2$ mechanism study (\S\ref{sec:exp-mech}) shows that the inversion tracks the data, not the architecture, implicating the non-Gaussian per-mode statistics that the theory idealizes away.
\end{enumerate}
\section{Related Work}
\label{sec:related}

\paragraph{Schr\"odinger bridges and bridge matching.}
Diffusion Schr\"odinger Bridge (DSB) approximates iterative proportional fitting
using score-based diffusion \cite{de2021diffusion}; DSBM and Iterative Markovian
Fitting instead alternate bridge matching with Markovian projection
\cite{shi2023diffusion,peluchetti2023diffusion}, and Iterative Proportional
Markovian Fitting unifies the two families \cite{kholkin2024diffusion,gushchin2024adversarial}. A parallel line pursues
simplified, non-iterative, or provably light solvers
\cite{tang2024simplified,gushchin2024light,peluchetti2023diffusion}, task-dependent state costs, and semi-supervised
couplings \cite{liu2024generalized,howard2026schrodinger}. These methods improve estimators built on a given reference. We instead ask which reference should be used.

\paragraph{Bridge models for restoration and inverse problems.}
\cite{liu20232} uses the analytic Brownian-bridge posterior to map paired degraded
and clean images, while \cite{zhou2024denoising} formulate DDBMs
over a general reference diffusion with VE/VP variants. Follow-up work improves
faster and non-Markovian sampling \cite{wang2024implicit,zheng2025diffusion},
few- or one-step distillation \cite{he2024consistency,gushchin2025inverse}, pretrained priors
\cite{wang2025irbridge}, mean-reverting and stochastic-control formulations
\cite{zhu2025unidb}, residual or energy-shortened trajectories
\cite{wang2026residual,hou2026energy}, and regularization against exposure bias
and distortion \cite{yao2025regularized}. These methods mainly modify the
drift or time horizon; the diffusion covariance remains isotropic,
with a scalar scale chosen by tuning.

\paragraph{Structured references and degradation-aware forward processes.}
Gaussian Schr\"odinger bridges admit closed forms
\cite{bunne2023schrodinger}, and richer reference dynamics such as multivariate
Ornstein--Uhlenbeck processes, topology-aware heat diffusions, analytic
linear--quadratic bridges---have been studied recently
\cite{zhang2026learning,wyrwal2026topological,chertkov2026analytic}. In these
works, the reference or its inducing cost is fixed in advance; we treat it as
unknown.

\paragraph{Spectral shaping and colored noise.}
The closest line of work studies the spectrum of injected noise. Fourier-space analyses of the forward process
\cite{falck2025fourier} motivate schedules that corrupt
all frequencies at a matched rate. Empirical studies report gains from blue
noise, low-frequency-heavy noise, spectrally anisotropic forward noise,
spectrally-guided or per-instance schedules, and frequency-band redistribution
at sampling time
\cite{huang2024blue,jiralerspong2025shaping,scimeca2025learning,
benita2026spectral,esteves2026spectrallyguided,davidson2026colored}. These works
show that noise color is an important design choice, but select it through
search or end-to-end fitting. This produces
dataset-specific rules without explaining why one spectrum should be preferred
or when color should have no effect. We provide both results in the bridge
setting: the optimal spectrum is the destroyed-information spectrum $P_k$, and
its effect provably vanishes in the exact-drift, schedule-adapted limit.

\paragraph{Schedule design, step budgets, and distortion--perception.}
Variational Diffusion Models and EDM established schedule and weighting choices
as important design variables \cite{kingma2021variational,karras2022elucidating}.
Later work studied ELBO-based reweighting, logSNR importance sampling, and
corrections to common schedules
\cite{kingma2023understanding,lin2024common,okada2024constant}. These methods
optimize the \emph{temporal} schedule of a fixed one-endpoint process and
evaluate it empirically under a given NFE budget. Our finite-step objective is
closed-form, so its optimum is proved rather than measured. Moreover, the
$z$-identity shows that reported color effects must be tested for schedule
confounds. Finally, the distortion--perception trade-off is classical
\cite{blau2018perception,freirich2021theory} and remains important in bridge
restoration \cite{yao2025regularized, fallah2025rareflow}. In our model, its factor-two cost follows
as a theorem, while the reference traces the entire curve.
\section{PRISM: A Theory of Reference Design}
\subsection{Which References Are Exactly Tractable?}
\label{sec:tractability}
Let \(B\) be standard Brownian motion in \(\R^d\), and let
\(\dd X_t = L(t)\,\dd B_t\) with \(Q(t) := L(t)L(t)^\top\).
Image-to-Image Schr\"odinger Bridge ($I^2$SB) methods \cite{liu20232, WANG2025111627} require two closed-form objects: the pinned marginal $X_t \mid (X_0,X_1)$ for training-state sampling and the reverse sub-bridge kernel $X_s \mid (X_t, X_0)$ for ancestral sampling. For any deterministic covariance schedule $Q(\cdot)$, both are multivariate Gaussians with explicit matrix formulas (Lemmas~\ref{lem:pinned}--\ref{lem:reverse}, App~\ref{app:tractability}). Practical $I^2$SB, however, needs an exact decomposition into independent scalar modes, each with its own spatial-frequency noise schedule. The following result characterizes this condition.
\begin{theorem}[Commuting-reference tractability]
\label{thm:tractability}
Suppose $Q(t)$ is real symmetric and positive semidefinite for almost every \(t\in[0,1]\). Then the following are equivalent:
\begin{enumerate}
\item
Pairwise commutativity: $Q(s)Q(t)=Q(t)Q(s)$
for all $s,t$ outside a common null set.
\item
A fixed modal basis: $ Q(t)=U\operatorname{diag} \bigl(q_1(t),\ldots,q_d(t)\bigr)U^\top$,
with a time-independent orthogonal $U$.
\item
In the coordinates $Y=U^\top X$, the reference is a product
of independent time-changed scalar Brownian motions:
$\dd Y_{k,t} = \sqrt{q_k(t)}\,\dd B_{k,t}$.
\end{enumerate}
Under these conditions, with $a_k(t) := \int_0^t q_k$, $v_k := a_k(1)$, and $\rho_k(t) := a_k(t)/v_k$, the pinned bridge decomposes exactly over modes:
\begin{equation}
\begin{aligned}
Y_{k,t}\mid(y_{k,0},y_{k,1})
&\sim \N\!\left(\mu_{k,t},\sigma_{k,t}^2\right),
\\
\mu_{k,t}
&=
\bigl(1-\rho_k(t)\bigr)y_{k,0}
+\rho_k(t)y_{k,1},
\\
\sigma_{k,t}^2
&=
v_k\rho_k(t)\bigl(1-\rho_k(t)\bigr),
\end{aligned}
\label{eq:pinned-modal}
\end{equation}
and for $0<s<t\le 1$ the reverse kernel is modewise Gaussian with mean coefficient $r_k(s,t) = \rho_k(s)/\rho_k(t)$ and variance $a_k(s)\bigl(1-r_k(s,t)\bigr)$.
\end{theorem}

\begin{corollary}[Schedule parameterization]
\label{cor:bijection}
The family in Theorem~\ref{thm:tractability} is parameterized exactly, per mode, by a total variance \(v_k>0\), called the \emph{color}, and an absolutely continuous nondecreasing schedule $\rho_k:[0,1]\to[0,1]$ with $\rho_k(0)=0$ and $\rho_k(1)=1$, called the \emph{allocation}, through $q_k(t)=v_k\dot{\rho}_k(t)$.
\end{corollary}

Three extensions matter in practice, all proved in App~\ref{app:tractability}.

(i) Static colors: references of the form $Q(t)=\beta(t)\Sigma$, used in prior fixed-color work, form the strict subfamily where all modes share one $\rho$; white $I^2$SB corresponds to $\Sigma=cI$.

(ii) Aliasing blocks: if every $Q(t)$ is block-diagonal under one fixed orthogonal decomposition, the bridge decouples into independent finite-dimensional blocks without within-block commutativity. This correctly models downsampling, where each frequency couples to its foldovers.

(iii) Commuting linear drift: simultaneously diagonalizable $F(t)$ and $Q(t)$ preserve exact modal tractability under modified Gaussian formulas, covering mean-reverting references; a Hilbert-space version holds under a trace condition.

\subsection{Solvable Model and Invisibility Principle}
\label{sec:invisibility}

Theorem~\ref{thm:tractability} separates the operator problem into scalar modes, so we work per mode and omit the index $k$. For each spatial frequency, the degradation model is
\begin{equation}
    x_0 \sim \N(0,S), \qquad x_1 = h\,x_0 + n, \quad n \sim \N(0,N).
\end{equation}
The posterior is $x_0 \mid x_1 \sim \N(Wx_1,P)$, where $W={Sh}/({h^2S+N})$ is the Wiener gain and $P={SN}/({h^2S+N})>0$ is the residual variance. For color $v\ge0$ and bridge level $\rho\in[0,1]$, the reference bridge is $x_\rho = (1-\rho)x_0 + \rho x_1 + \sqrt{v\rho(1-\rho)}\,\xi$.

Let $0=\rho_0<\dots<\rho_T=1$ be the discretization grid. The \emph{plug-in-mean ancestral sampler}---standard $I^2$SB with a Bayes-optimal network---starts at $x_1$. At each level, it computes $\hat x_0=\E[x_0\mid x_{\rho_i},x_1]$ and applies the reverse kernel from Theorem~\ref{thm:tractability}, replacing $x_0$ with $\hat x_0$.

Define $\varphi(\rho) := (1-\rho)P + v\rho$ and
\begin{equation}
K(\rho) = \frac{P}{\varphi(\rho)}, \quad
M(\rho) = \frac{v\rho P}{\varphi(\rho)}, \quad
\Sigma(\rho) = (1-\rho)\varphi(\rho).
\label{eq:phi}
\end{equation}
Here $K(\rho)$ is the conditional gain, $M(\rho)=\Var(x_0\mid x_\rho,x_1)$ is the conditional variance, and $\Sigma(\rho)=\Var(x_\rho\mid x_1)$. Given \(x_1\), every sampler state remains Gaussian: $x_{\rho_i}\mid x_1\sim\N(c_i x_1,V_i)$, with \(c_i\) and \(V_i\) following an explicit affine recursion (App~\ref{app:thm0}).
\begin{theorem}[Invisibility of the reference]
\label{thm:invisibility}
(a) Fix $v>0$ and any grid whose deficit sum in Lemma~\ref{lem:deficit-closed} (App~\ref{app:thm0}) vanishes as $T\to\infty$, including uniform grids and power grids $\rho_i=(i/T)^a$ with $a\ge1$. Then the terminal law converges to the true posterior: $\N(c_0x_1,V_0)\longrightarrow\N(Wx_1,P)$.
More precisely, $D_0/P\to 0$ at rate $O(\log T/T)$ on uniform grids and $O(1/T)$ on power grids with $a>1$. Because the terminal mean is exact, $\KL = {D_0^2}/{4P^2}\,\bigl(1+o(1)\bigr)$; consequently, $\KL = O\!\left(\left({\log T}/{T}\right)^2\right)$ on uniform grids and $\KL=O(1/T^2)$ on power grids with $a>1$.
The limiting law is independent of $v>0$ and the admissible schedule; the reference becomes invisible in the infinite-step limit.
(b) At the boundary $v=0$, the terminal law is $\delta(Wx_1)$
for every finite $T$: the sampler collapses discontinuously to Wiener regression and has infinite KL divergence from the posterior.
\end{theorem}

The proof (App~\ref{app:thm0}) rests on three exact properties. First, \emph{the mean is exact at every finite $T$}: for every grid, $T$, and $v\ge 0$, $c_i=(1-\rho_i)W+\rho_i$. Replacing $x_0$ by its posterior mean removes only its conditional randomness, so finite-step error affects only variance. 

Second, \emph{the contraction telescopes exactly}: the step coefficient is $A_i={\varphi(\rho_{i-1})}/{\varphi(\rho_i)}$, so products of step coefficients reduce to exact ratios of $\varphi$. Third, the \emph{variance deficit} $D_i:=\Sigma(\rho_i)-V_i$ satisfies a linear recursion with a nonnegative source term, yielding:
\begin{corollary}[Systematic underdispersion]
\label{cor:underdispersion}
Suppose $v>0$. Then $V_i<\Sigma(\rho_i)$ for every $i<T$; in particular, $V_0<P$ for every finite $T$. The plug-in-mean sampler has a single failure mode: it produces too little variance.
\end{corollary}

Theorem~\ref{thm:invisibility} is the central negative result. Any meaningful objective for reference design must assign a cost to a finite resource: solver steps (\S\ref{sec:finitestep}) or model accuracy and data (\S\ref{sec:modelerror}). Without such a cost, objectives become degenerate; for example, the integrated Bayes risk of the training target is minimized at the collapsed boundary $v=0$ 
(App~\ref{app:thm0}).

\subsection{Steps: The Exact Finite-Step Theory}
\label{sec:finitestep}
\subsubsection{The exact objective and the form of optimizer}
Unrolling the variance-deficit recursion with the telescoped products gives the
deficit in closed form (Lemma~\ref{lem:deficit-closed}, App~\ref{app:thm0}):
\begin{equation}
\begin{aligned}
D_0(v,P;\text{grid})  =
vP^3\sum_{i=1}^{T}
\frac{(\Delta\rho_i)^2}
{\rho_i\,\varphi(\rho_i)\,
 \varphi(\rho_{i-1})^2},
\\
u
=
1-{D_0}/{P},
\qquad
\mathrm{KL}
=
{1}/{2}\bigl(u-1-\ln u\bigr).
\end{aligned}
\label{eq:deficit}
\end{equation}
This is the exact $T$-step objective. Reference design minimizes $\sum_k \KL_k$ over colors $\{v_k\}$ and, optionally, per-mode schedules (Cor.~\ref{cor:bijection}).

\begin{theorem}[Scale symmetry and proportionality]
\label{thm:proportional}
Fix any shared grid. The per-mode $\KL$ depends on $(v,P)$ only through the dimensionless ratio $x=v/P$: a mode-independent profile $\Phi$ satisfies $\KL(v,P;\text{grid})=\Phi(v/P)$. Since $\Phi(x)\to\infty$ as $x\to0$ or $x\to\infty$, interior optimizers exist. The objective also separates: $\sum_k\KL_k=\sum_k\Phi(v_k/P_k)$.

Consequently:

(i) At every local optimizer, $v_k = x_k P_k$, where each $x_k$ is a
local minimizer of the \emph{same} mode-independent profile $\Phi$.

(ii) At every global optimizer, $x_k \in \operatorname{arg\,min}\Phi$
for every $k$.

(iii) Whenever $\operatorname{arg\,min}\Phi$ is a singleton, every
global optimizer is proportional to the posterior spectrum,
\begin{equation}
v_k^{*}
=
x^{*}(\mathrm{grid},T)\,P_k,
\qquad k=1,\dots,d,
\label{eq:optimal-colour}
\end{equation}
with one mode-independent $x^{*}$ for every $T$. If the minimizer is nonunique, as on some non-uniform grids (Prop.~\ref{prop:uniqueness}), modes may select different local minimizers, yielding non-proportional \emph{local} optima.
\end{theorem}

This proportionality follows from an exact algebraic symmetry, not an asymptotic argument. In the per-mode Gaussian problem, $P$ is the only scale and $v/P$ the only dimensionless parameter; $T$ controls only the constant.

\begin{proposition}[The proportionality constant]
\label{prop:constant}
On the uniform grid, every minimizer satisfies
\begin{equation}
    x^{*}(T) = (2\ln T)^{-1/2}\bigl(1+o(1)\bigr),
\label{eq:rule}
\end{equation}
and the optimal value satisfies $\KL(x^{*})=\frac{\ln T}{2T^2}\bigl(1+o(1)\bigr)$.

The proof (App~\ref{app:thmB}) establishes the uniform expansion
$T\,{D_0}/{P}=x\ln T+\frac{1}{2x}+1+o(1)$
on the window $x\in[x_0/3,\,3x_0]$ with $x_0 = (2\ln T)^{-1/2}$, together with
global lower bounds excluding minimizers outside the window. A numerical
certificate with exact derivatives independently corroborates the expansion;
the finite-$T$ remainder stabilizes near $1.15$ over $T=10^3$--$10^5$.
\end{proposition}

Practically, $x^{*}\approx0.2$--$0.6$ for realistic step counts: the optimal
reference injects a modest fraction of the destroyed-information spectrum, and
that fraction shrinks only like $(\ln T)^{-1/2}$.

\begin{proposition}[Uniqueness is grid-dependent]
\label{prop:uniqueness}
At $T=2$, $\Phi$ is strictly unimodal. On uniform grids, uniqueness is proved for $2\le T\le120$ and numerically certified for $T\le200$ and $T\in\{500,1000,5000\}$ (App~\ref{app:thmB}); extending the proof to general $T$ remains an open problem.

On general grids, uniqueness is false. An analytic two-cluster construction shows that, already at $T=3$, the grid $\rho=\bigl({1}/({1+C}),1/2,1\bigr)$ yields at least two local minima of $\Phi$ for every $C\ge100$ (App~\ref{app:thmB}). As $C\to\infty$, the profile splits into two unimodal $T=2$ profiles, one per cluster, at scales $x\asymp1$ and $x\asymp C$. A $T=48$ two-cluster grid has \emph{seven} local minima: each step cluster favors its own noise scale.
\end{proposition}

The counterexample carries a warning: on non-uniform grids,
\emph{schedule design and color design cannot be decoupled}.

\begin{proposition}[Budgets bend the exponent]
\label{prop:budget}
Fix $P_1,\dots,P_d$ and a total-noise budget $V$ with
$V P_{\max}/\sum_j P_j^2 \le \tfrac34$. On uniform grids, as $T\to\infty$,
every global optimizer of $\sum_k\KL_k$ subject to $\sum_k v_k = V$ satisfies
\begin{equation}
v_k^{*}
=
\frac{P_k^2}{\sum_j P_j^2}\,V\,
\left(1+O\!\left(\frac{1}{\ln T}\right)\right),
\qquad k=1,\dots,d.
\label{eq:budget-law}
\end{equation}
The fixed budget places every mode on the increasing branch of $\Phi$: the free optima $x^{*}(T)P_k$ vanish as $T\to\infty$, but the budget remains fixed. The optimal exponent therefore bends from $P_k$ to $P_k^2$. Equal-budget comparisons, common in empirical color sweeps, silently change the optimization problem and steepen the optimum. We report both conventions in \S\ref{sec:experiments} because they answer different questions.
\end{proposition}

\subsubsection{The $z$-identity: color and schedule are exchangeable}
The key structure appears under the change of variables $z(\rho):={v\rho}/{\varphi(\rho)}={M(\rho)}/{P}$, the normalized conditional-MMSE level, with $z_0=0$ and $z_T=1$.
\begin{theorem}[Exact $z$-identity and schedule theory]
\label{thm:zidentity}
For every color $v>0$ and every grid,
\begin{equation}
    D_0
    =
    P\sum_{i=1}^{T}
    {(z_i-z_{i-1})^2}/{z_i}.
    \label{eq:z-identity}
\end{equation}
Consequently:

(a) Color-freeness.
For every $v>0$, the map $\rho\mapsto z(\rho)$ is a bijection of $[0,1]$. Every
color therefore achieves the same set of $z$-grids, and hence the same
schedule-optimized deficit, at every finite $T$: under per-mode schedule
adaptation, discretization error cannot prefer any color.

(b) Unique optimal schedule.
The objective $F(z)=\sum_{i=1}^{T}(\Delta z_i)^2/z_i$ is jointly convex, being
a sum of quadratic-over-linear terms, and its minimizer over pinned grids is
unique.

(c) Explicit forward recurrence.
With $a_i = z_{i-1}/z_i$, the first-order conditions become
$2a_{i+1}=1+a_i^2$, $a_1=0$, giving an explicit forward recurrence
($a_2=\tfrac12$, $a_3=\tfrac58$, $a_4=\tfrac{89}{128},\ \ldots$), with no
shooting required.

(d) The floor and the continuum limit.
The optimal deficit satisfies
$D_0^{*}=\frac{4P}{T}\bigl(1+O({\log^2 T}/{T})\bigr)$, giving a per-mode KL floor $\sim 4/T^2$. The exact optimizer converges to the
continuum schedule $z = s^2$ at the explicit rate
$z_i^{*}=(i/T)^2\exp\bigl(O({\ln(i+1)}/({i+1}))\bigr)$;
in the original variable,
$\rho^{*}(s)={Ps^2}/\bigl({Ps^2+v(1-s^2)}\bigr)$,
and the conditional MMSE is quadratic in solver time: $M(\rho^{*}(s)) = Ps^2$.
\end{theorem}

Theorem~\ref{thm:zidentity} reframes the design space: the reference affects discretization only through its $z$-schedule, while per-mode color is merely a reparameterization. Two practical consequences follow. With a \emph{shared} schedule, as in existing implementations, Theorem~\ref{thm:proportional} determines the color optimum. With frequency-adaptive schedules, discretization is exactly color-free; any color effect must instead arise from learned drift or finite data.

\subsection{Model Error: Where Color Genuinely Lives}
\label{sec:modelerror}
In the solvable model, the exact drift is affine. Any learned affine drift therefore differs by two quantities per level: relative \emph{gain error} $\eta_i$ and additive \emph{bias} $\beta_i$. Since everything is conditioned on $x_1$, the $\beta_i$ may depend on it; this absorbs errors in the drift's $x_1$ coefficient. This two-channel decomposition is exhaustive within the model, and both channels preserve affine-Gaussian structure, making all terminal laws below exact.

\begin{proposition}[Gain error cannot bias the sampler]
\label{prop:gain-unbiased}
With arbitrary gain errors $\{\eta_i\}$, the terminal conditional mean is still
exactly $Wx_1$. Gain error acts purely on the variance channel; only bias moves
the mean.
\end{proposition}

\begin{theorem}[Exact $z$-reduction of the perturbed sampler]
\label{thm:colorblind}
Work at a fixed $z$-grid with per-level error profile $(\eta(z),\beta(z))$. The
perturbed contraction is
$\hat A_i=A_i\bigl(1+\eta_i\frac{\Delta z_i}{z_i}\bigr)$.
A bias injected at level $i$ reaches the output with weight
$\frac{\Delta z_i}{z_i}\prod_{m<i}\bigl(1+\eta_m\frac{\Delta z_m}{z_m}\bigr)$, which is exactly
$\Delta z_i/z_i$ when $\eta\equiv0$. Thus, the terminal mean error is
\begin{equation}
m_0
=
\sum_{i=1}^{T}
\beta_i\,
\frac{\Delta z_i}{z_i}
\prod_{m<i}
\left(
1+\eta_m\frac{\Delta z_m}{z_m}
\right),
\label{eq:perturbed-mean}
\end{equation}
and the terminal variance has the closed form
\begin{equation}
\begin{aligned}
V_0
&=
P\sum_{j=1}^{T}
z_{j-1}\frac{\Delta z_j}{z_j}
\times
\prod_{m<j}
\left(
1+\eta_m\frac{\Delta z_m}{z_m}
\right)^2,
\end{aligned}
\label{eq:perturbed-variance}
\end{equation}
whose $\eta\equiv0$ case recovers $V_0/P = 1 - \sum_i(\Delta z_i)^2/z_i$.
Every factor depends only on the $z$-grid and error profile. Thus, the sampler's terminal law is identical for every color $v>0$: given the $z$-schedule, sampling remains color-blind under arbitrary drift error. The terminal KL is
\begin{equation}
u={V_0}/{P},
\qquad
\KL
=
{1}/{2}
\left(
u+{m_0^2}/{P}-1-\ln u
\right).
\label{eq:perturbed-kl}
\end{equation}
\end{theorem}

\begin{proposition}[The learning problem is also color-blind, per mode]
\label{prop:learning-null}
At level $z$, the per-mode joint law of $(x_0, x_t)$ given $x_1$ is determined
up to scale by $z$ alone: $\mathrm{corr}^2(x_0,\,\text{deviation}\mid x_1)
= 1 - z$ exactly. Hence any scale-equivariant learner has a color-free
relative-error law at each $z$-level; for example, zero-intercept OLS from
$n$ pairs has $\Var(\hat K/K - 1) = z/\big(n(1-z)\big)$ on a normalized
design $\sum_i X_i^2 = n\tau^2$ (random design replaces $n$ by $n-2$).
With Theorem~\ref{thm:colorblind}, color is therefore invisible end to end in the per-mode scalar model, given the $z$-schedule.
\end{proposition}
\begin{figure*}[!htb]
    \centering
    \includegraphics[width=\columnwidth]{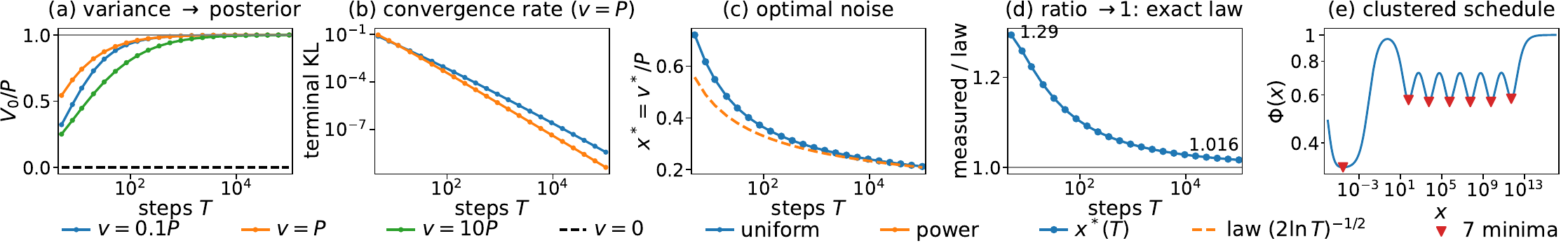}
    \caption{\textbf{Exact finite-step predictions.} (a) Invisibility
(Theorem~\ref{thm:invisibility}). (b) Terminal KL rates, uniform vs.\ power
grid. (c,d) The constant (Prop.~\ref{prop:constant}): measured $x^{*}(T)$
against $(2\ln T)^{-1/2}$, and their ratio. (e) Reference multimodality
(Prop.~\ref{prop:uniqueness}): seven local minima of $\Phi$ on a
clustered $T{=}48$ schedule, certified in 50-digit arithmetic.}
    \label{fig:exact}
\end{figure*}
Color therefore matters only through what the per-mode scalar model cannot see:
(a) scale-non-equivariant learning, and (b) cross-mode
coupling. Mechanism (a) is present in every trained network (weight decay,
ridge penalties, finite initialization scales), and its minimal model already
breaks proportionality in a definite direction:

\begin{proposition}[Ridge whitens the optimal reference]
\label{prop:ridge}
A shared ridge penalty $\lambda$ induces relative shrinkage $\eta_k(z)=-{\lambda}/({\lambda+n\Sigma_k(z)})$ (same convention), a \emph{dimensionful} bias that breaks the scale symmetry protecting $v\propto P$. Exact-chain optimization yields a sharply decreasing $x^{*}(nP)$: at $T{=}100$ and $\lambda{=}1$, it falls from $11.92$ at $nP{=}10$ to $0.363$ as $nP\to\infty$. Weakly observed frequencies therefore need disproportionately more reference noise. The law bends sub-proportionally, $v_k^{*}=x^{*}(nP_k)P_k$, so \emph{regularization whitens the optimal reference}; pure $v\propto P$ returns only in the equivariant, infinite-data limit.
\end{proposition}

Two further results follow from this machinery (proofs in
App~\ref{app:thmC}). \textbf{Distortion--perception trade-off:} the
posterior-mean estimator has MSE exactly $P$, independent of the
reference, while the sampled output has MSE $= P + V_0 \in [P, 2P)$; the
classical factor-2 cost of posterior sampling thus emerges as a theorem, with
the reference tracing the entire trade-off curve. \textbf{Principled
underdispersion correction:} since the sampler's main failure is
underdispersion (Corollary~\ref{cor:underdispersion}), deliberate gain
inflation is a calibrated fix, with an exact condition: choose $\eta(z)$ so
that the reweighted sum in \eqref{eq:perturbed-variance} reaches $P$, the
closed-form counterpart of temperature and churn heuristics. Theorem~\ref{thm:colorblind} also yields a \textbf{measurement protocol} for
real networks: probe the trained drift's Jacobian along the path to estimate
$\hat\eta(z)$, then \emph{predict} the sampler's terminal variance via
\eqref{eq:perturbed-variance}, without retraining, just one forward analysis. We use
this protocol in \S\ref{sec:experiments}.
\section{Experiments}
\label{sec:experiments}
\subsection{Exact Numerical Evaluation}
\label{sec:exp-exact}
The deterministic affine recursions of
Sections~\ref{sec:invisibility}--\ref{sec:modelerror} isolate
finite-step effects from sampling and training error; every quantity is
computed in closed form and cross-checked to machine precision
(App~\ref{app:exact-details}).

Figure~\ref{fig:exact} confirms each prediction. For $P{=}v{=}4$, the
terminal variance reaches $3.9995$ at $T{=}10^5$, while $v=0$ collapses
(Theorem~\ref{thm:invisibility}). The measured optimum tracks the
$(2\ln T)^{-1/2}$ law, with the ratio falling from $1.29$ at $T{=}5$ to
$1.017$ at $T{=}10^5$ (Prop.~\ref{prop:constant}), and the clustered
grid exhibits the predicted multimodality
(Prop.~\ref{prop:uniqueness}). In a two-mode problem, the matched
allocation minimizes KL at every tested $T$; at $T{=}200$, the
Bayes-risk vertex has $294\times$ larger KL
(Table~\ref{tab:twomode} in App~\ref{app:exact-details}). Optimized $z$-schedules remove
reference-color dependence to numerical precision
(Theorem~\ref{thm:zidentity}(a)), and the budgeted optimum moves toward the
predicted $P^2$ allocation (Prop.~\ref{prop:budget}).

\subsection{Learned Gaussian Models}
\label{sec:exp-gaussian}
We train per-mode predictors on $64$--$256$ independent modes with
$S_k\sim k^{-2}$, a Gaussian modulation transfer function, and flat
observation noise, comparing white, matched ($v\propto P$),
anti-matched ($v\propto1/P$), and prior-colored ($v\propto S$)
references at equal total budget ($3$ seeds, common random numbers).
\begin{figure}[!htb]
    \centering
    \includegraphics[width=.6\columnwidth]{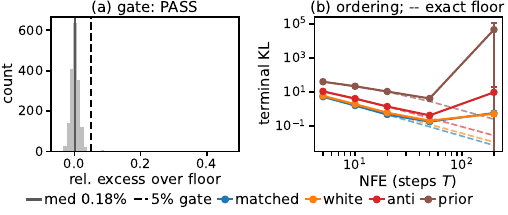}
    \caption{\textbf{Learned Gaussian models.} (a) Relative excess of
    the converged loss over the closed-form Bayes floor for
    $\rho\in[0.1,0.95]$: median $0.18\%$ against the pre-registered
    $5\%$ gate. (b) NFE-vs-KL by reference ($3$ seeds, mean$\pm$std;
    dashed lines are exact-drift floors).}
    \label{fig:gaussian}
\end{figure}
The learned models reach the theory's floor and reproduce its ordering
(Fig.~\ref{fig:gaussian}). After $12$k steps, the median relative
excess over the closed-form Bayes floor is $0.18\%$, well inside the
pre-registered $5\%$ gate. 
At NFE $50$, total KL runs from $0.19{\pm}0.03$ (matched) to
$4.3{\pm}0.3$ (prior-colored), with white and anti-matched
between---the exact-drift ordering, at learned drifts.
At NFE $200$, gain error (Prop.~\ref{prop:gain-unbiased}) dominates
discretization error: matched and white overlap within seed variation, while the colored references become unstable
(App~\ref{app:gaussian-details}).

The measured optimal scale agrees with theory where the model is
accurate: $0.575$ vs.\ $0.577$ at $T{=}10$ and $0.402$ vs.\ $0.404$ at
$T{=}50$. At $T{=}200$ it sits $8\%$ above the exact-drift value, the
direction predicted for finite-data regularization by
Proposition~\ref{prop:ridge}; the ridge sweep, the Jacobian
probe (median $0.6\%$ relative error in predicted terminal variance,
with no retraining), and the per-mode-schedule null are in
App~\ref{app:gaussian-details}.

\begin{table}[!htb]
\small
    \centering
    \caption{\textbf{FFHQ restoration at NFE $50$} (original recipe;
    $\sigma_{\mathrm{blur}}{=}2.0$, $\sigma_n{=}0.05$). FID from $50$k samples; white and matched use three
    seeds (mean$\pm$std), others one, so bold gaps ${<}0.1$ are within
    seed noise.}
    \label{tab:ffhq}
    \setlength{\tabcolsep}{3.5pt}
    \begin{tabular}{lcccc}
        \toprule
        reference & PSNR$\uparrow$ & SSIM$\uparrow$ & LPIPS$\downarrow$ & FID$\downarrow$ \\
        \midrule
        white & $25.01$ & $0.836$ & $\mathbf{0.0337}$ & $\mathbf{7.87}{\pm}0.06$ \\
        matched $v\propto P$ & $24.98$ & $0.836$ & $0.0344$ & $8.45{\pm}0.13$ \\
        matched, equal budget & $24.97$ & $0.836$ & $0.0370$ & $9.12$ \\
        anti-matched & $\mathbf{25.75}$ & $\mathbf{0.854}$ & $0.0426$ & $14.14$ \\
        fixed color ($\alpha{=}1$) & $25.04$ & $0.837$ & $0.0338$ & $7.93$ \\
        fixed color ($\alpha{=}2$) & $25.21$ & $0.840$ & $0.0354$ & $8.46$ \\
        \bottomrule
    \end{tabular}
\end{table}

\begin{figure*}[!htb]
    \centering
    \includegraphics[width=\columnwidth]{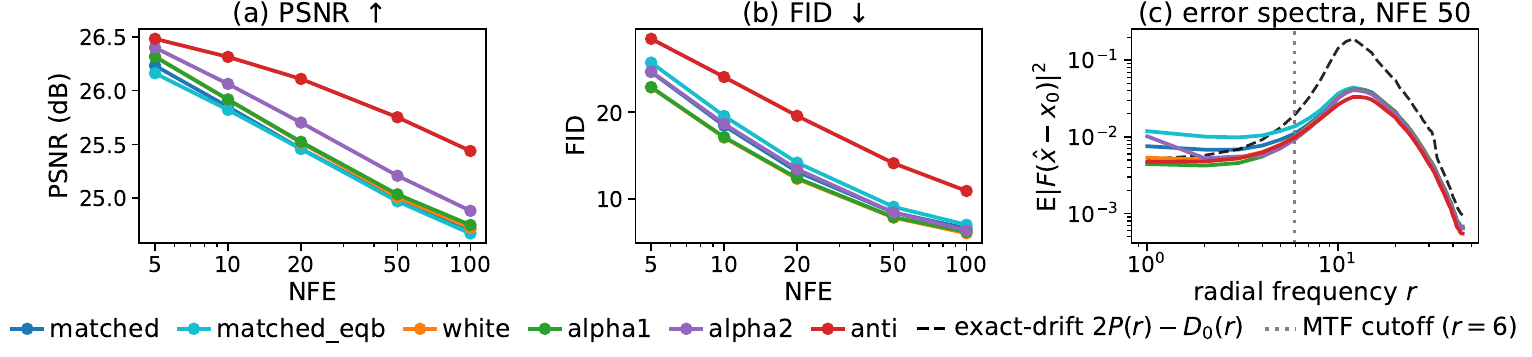}
    \caption{\textbf{FFHQ under the original recipe} (common random numbers;
    matched and white mean$\pm$std over $3$ seeds). (a) PSNR and (b) FID
    against NFE: anti-matched attains the best distortion and the worst
    perception at every NFE---the distortion--perception trade-off of
    App~\ref{app:thmC}. (c) Radially averaged error spectra at NFE
    $50$ against the exact-drift prediction $2P(r)-D_0(r)$ (dashed);
    reference differences concentrate near the MTF cutoff (SSIM, LPIPS,
    error-ratio view, training fingerprint: App~\ref{app:ffhq-details}).}
    \label{fig:ffhqcurves}
\end{figure*}

\subsection{FFHQ Under a Known Degradation}
\label{sec:exp-ffhq}

We train bridge models on FFHQ at $64{\times}64$ with known Gaussian blur and flat observation noise; $P_k$ and the matched reference are
therefore fixed by the theory rather than tuned. All references share one training recipe, one solver, and common evaluation randomness (App~\ref{app:ffhq-details}).

Two predictions transfer exactly. First, the
distortion--perception trade-off: anti-matched gives the best PSNR and SSIM and the worst LPIPS and FID at every NFE
(Table~\ref{tab:ffhq}, Fig.~\ref{fig:ffhqcurves}), as the theory
predicts for a reference that suppresses variance where the sensor is blind: the posterior mean sharpens while sampled detail is lost (App~\ref{app:thmC}). Second, the predicted spectral localization:
reference-dependent error differences concentrate near the MTF cutoff, where $P(k)$ transitions (Fig.~\ref{fig:ffhqcurves}(c)).

The fine-grained matched-first ordering does not transfer. White and
the fixed $\alpha{=}1$ reference attain lower FID and LPIPS than
matched over the tested NFE range (Table~\ref{tab:ffhq}). Re-evaluating every trained network
on the $z$-optimal schedule of Theorem~\ref{thm:zidentity} improves all
references by a similar amount and leaves the white--matched gap nearly
unchanged ($0.58\to0.53$; App~\ref{app:ffhq-details}), so shared
discretization error does not explain the gap. By
Theorem~\ref{thm:colorblind}, what remains is the learned drift: either
scale-non-equivariant learning (Prop.~\ref{prop:ridge}) or cross-mode
coupling. The next two studies separate them.

\begin{table}[!htb]
\small
    \centering
    \caption{\textbf{FFHQ under the converged recipe}
    (dropout $0$, batch $512$, $300$k steps; $v\propto P_k^{\theta}$ at
    fixed budget). FID from $50$k samples; endpoints mean$\pm$std over
    $5$ seeds, $\theta$ one seed.}
    \label{tab:conv}
    \setlength{\tabcolsep}{4pt}
    \begin{tabular}{lccc}
        \toprule
        reference & FID@$10\downarrow$ & FID@$50\downarrow$ & LPIPS@$50\downarrow$ \\
        \midrule
        white ($\theta{=}0$) & $9.62{\pm}0.06$ & $\mathbf{3.53}{\pm}0.02$ & $\mathbf{0.0323}$ \\
        $\theta{=}0.25$ & $\mathbf{9.42}$ & $3.58$ & $0.0326$ \\
        $\theta{=}0.5$ & $9.75$ & $3.78$ & $0.0332$ \\
        $\theta{=}0.75$ & $10.77$ & $4.15$ & $0.0343$ \\
        matched ($\theta{=}1$) & $12.05{\pm}0.15$ & $4.71{\pm}0.09$ & $0.0355$ \\
        \bottomrule
    \end{tabular}
\end{table}

\subsection{Training Regime and Reference Exponent}
\label{sec:exp-regime}
Ridge whitening (Prop.~\ref{prop:ridge}) predicts that improved optimization should shift the optimal reference toward $v\propto P_k$. We test it with a pre-registered
intervention and a
sweep $v\propto P_k^{\theta}$, $\theta\in\{0,\tfrac14,\tfrac12,\tfrac34,1\}$. 
The converged recipe improves both matched and white in absolute terms, but it increases the matched--white FID gap on the same $50$k evaluation set: the gap grows from $1.36$ to $2.43$ at NFE~$10$ and from $0.58$ to $1.18$ at NFE~$50$. The endpoint prediction is therefore refuted.

Stronger blur ($\sigma_{\mathrm{blur}}{=}4.0$) widens the gap, and replication on CelebA \cite{celebA} under the same experimental protocol reproduces both the white-first ordering and the distortion--perception trade-off (App~\ref{app:celeba}).
The sweep reveals budget dependence. At low NFE, mild spectral coloring is beneficial: $\theta{=}0.25$ achieves the lowest observed FID for NFE~$\le20$ (Table~\ref{tab:conv}). As the budget increases, however, the optimum shifts toward white, which performs best on all reported metrics for NFE~$\ge50$. Thus, the experiments support partial $P_k$-coloring in the few-step regime.
They leave open whether the practical failure of the per-mode theory arises from dependence among Fourier modes, non-Gaussian per-mode marginals, or both.
\subsection{Mechanism Study: Data vs. Coupling}
\label{sec:exp-mech}
\textbf{Non-Gaussianity, rather than mode coupling, explains the inversion.}
Our $2{\times}2$ study separates model coupling from data statistics. The white-over-matched inversion persists whenever real FFHQ statistics are retained, even when coupling is impossible. On a Gaussian corpus with the same spectrum but no higher-order cross-mode structure, the matched reference regains its low-NFE advantage: FID improves from $19.3$ to $13.5$ at NFE~$5$ and from $7.9$ to $5.3$ at NFE~$10$. At NFE~$50$, the two references are nearly tied. The coupling probe shows substantial off-diagonal Fourier response in FFHQ-trained U-Nets, roughly tenfold less in Gaussian-trained models, and numerical zero in per-mode models.
The inversion tracks the data, not
the coupling: it persists in per-mode models where coupling is
numerically zero, and it vanishes when a coupling-capable U-Net is
trained on Gaussian data. Coupling in FFHQ-trained networks is
therefore a symptom of non-Gaussian image statistics, not the operative
cause; the per-mode Gaussianity assumption, not the decoupling
assumption alone, is what fails on real images. The near-tie at NFE~$50$ is consistent with the exact-drift invisibility
limit of Theorem~\ref{thm:invisibility}, in which reference gaps vanish as
the step budget grows. The theorem does not guarantee this for learned,
non-Gaussian models, where drift error need not vanish with NFE; indeed,
in \S\ref{sec:exp-gaussian} the colored references destabilize at NFE~$200$. 

\section{Conclusion}
We have developed a theory of reference design for Schr\"odinger bridge models. With exact drift and unlimited steps the reference is invisible; under a finite budget every optimum is proportional to the destroyed-information spectrum $P_k$ with color and schedule exactly exchangeable.
On real images this ordering inverts, and controlled studies trace the inversion to non-Gaussian per-mode statistics. The resulting recipe is budget-dependent: mild $P_k$ coloring in the few-step regime, white noise otherwise, with the $z$-optimal schedule in both cases. Extending the theory beyond Gaussian conditional modes is therefore the main direction for future work.

\newpage

\section{Appendix}
\subsection{Overview}
\label{app:overview}

Apart from the six textbook facts collected in Table~\ref{tab:standard}, every argument is proved here from first principles. Sections~\ref{app:tractability}--\ref{app:thmC} present the proofs in the order in which the results appear in the main text.
Section~\ref{app:experiments} provides the experimental protocols and the additional results.

Table~\ref{tab:notation} collects the symbols. Three conventions are
in force everywhere below.

\begin{enumerate}[leftmargin=1.4em,itemsep=1pt,topsep=2pt]
\item \emph{Conditioning on the observation.} All statements about the
  sampler are conditional on $x_1$. Expectations, variances and KL
  divergences are conditional on $x_1$ unless said otherwise.
\item \emph{Per-mode reduction.} After Theorem~\ref{thm:tractability}
  we work in the diagonalizing basis and drop the mode index $k$
  whenever a statement is about a single mode. Sums over $k$ reappear
  only when modes are coupled by a budget
  (App~\ref{app:budget}) or by a shared learner
  (App~\ref{app:ridge}).
\item \emph{Grids.} A grid is $0=\rho_0<\rho_1<\dots<\rho_T=1$, and
  $i$ runs from $T$ (the observation end) down to $0$ (the clean end),
  which is the direction the sampler travels.
\end{enumerate}

\begin{table}[!htbp]
\small\centering
\caption{\textbf{Notation.} Deliberate overloads are marked
$\dagger$; they never occur in the same argument.}
\label{tab:notation}
\begin{tabularx}{\columnwidth}{@{}l>{\raggedright\arraybackslash}X@{}}
\toprule
\multicolumn{2}{@{}l}{\emph{Signals, spectra, modes}}\\
$x_0,\,x_1$ & clean signal; observation \\
$X_t$ & bridge state, $t\in[0,1]$ \\
$U$, $k$ & orthogonal diagonalizer; mode index \\
$S_k,\,h_k,\,N_k$ & prior power, transfer function, observation noise \\
$P_k$ & posterior variance of mode $k$ given $x_1$ \\
$W$ & posterior-mean (Wiener) operator, $\E[x_0\mid x_1]=Wx_1$ \\
$v_k$ & reference variance of mode $k$ --- the \emph{color} \\
$V$ & total budget $\sum_kv_k$ (App~\ref{app:budget}) $\dagger$ \\
\midrule
\multicolumn{2}{@{}l}{\emph{Reference process}}\\
$Q(t)$ & instantaneous increment covariance \\
$A(t)$ & $\int_0^tQ$; $V=A(1)$ $\dagger$ \\
$q_k,a_k$ & modal versions of $Q,A$; $v_k=a_k(1)$ \\
$\rho_k(t)$ & schedule $a_k(t)/v_k$, increasing from $0$ to $1$ \\
$\beta(t)$ & rate of a static-color reference (Cor.~\ref{cor:prism}) \\
\midrule
\multicolumn{2}{@{}l}{\emph{Grids and coordinates}}\\
$T$ & number of sampler steps (NFE) \\
$\rho_i,\Delta\rho_i$ & grid level; $\rho_i-\rho_{i-1}$ \\
$r_i$ & $\rho_{i-1}/\rho_i$ \\
$x$ & $v/P$, the only dimensionless color parameter \\
$\varphi(\rho)$ & $(1-\rho)P+v\rho=P\psi(\rho)$ \\
$\psi(\rho)$ & $1+(x-1)\rho$; $b:=1-x$ \\
$M(\rho)$ & conditional MMSE at level $\rho$ \\
$z(\rho)$ & $v\rho/\varphi(\rho)\in[0,1]$, the solver coordinate \\
$z_i,\Delta z_i$ & $z(\rho_i)$; $z_i-z_{i-1}$ \\
$a_i,\varepsilon_i$ & $z_{i-1}/z_i$; $1-a_i$ (App~\ref{app:zidentity}) \\
$s$ & continuum solver time, $z=z(s)$ \\
\midrule
\multicolumn{2}{@{}l}{\emph{Sampler, deficit, objective}}\\
$a(\xi)$ & frontier launch point, root of $\ln\tfrac1a-(1-a)=\xi$ (Theorem~\ref{thm:frontier}) \\
$x_{\rho_i}$ & sampler state at level $\rho_i$, conditional on $x_1$ \\
$\xi,\,\xi_i$ & standard Gaussian noise (bridge; sampler step $i$) $\dagger$ \\
$K(\rho)$ & regression gain of the plug-in drift \\
$A_i,B_i,q_i$ & affine-recursion coefficients $\dagger$ \\
$c_i,V_i$ & mean and variance coefficients of the plug-in chain $\dagger$ \\
$\Sigma(\rho)$ & variance of the exact ancestral chain \\
$D_i,\,D_0$ & variance deficit $\Sigma(\rho_i)-V_i$; its terminal value \\
$H(x)$ & $D_0/P$ \\
$F(z)$ & $\sum_i(\Delta z_i)^2/z_i$; equals $H$ under $z=z(\rho)$ \\
$\Phi(x)$ & terminal KL, $\tfrac12(u-1-\ln u)$ with $u=1-H$ \\
$\Har_T$ & harmonic number $\sum_{i\le T}1/i$ \\
$\chi_T$ & $(2\ln T)^{-1/2}$, the target point \\
$\Wind$ & localization window $[\chi_T/3,3\chi_T]$ \\
$\eta_i,\beta_i$ & relative gain error; additive bias at step $i$ \\
$\Xi,\,\xi$ & gain-error susceptibility; its budget $\dagger$ \\
$\lambda$ & ridge penalty (App~\ref{app:ridge}); Lagrange multiplier (App~\ref{app:budget}) \\
\bottomrule
\end{tabularx}
\end{table}

Table~\ref{tab:standard} lists the six textbook facts we use, with the
tags (S1)--(S6) by which they are cited throughout. Cauchy--Schwarz
(Theorem~\ref{thm:frontier}(ii)) and the mean value theorem
(Theorem~\ref{thm:invisibility}(a)) are used only in their elementary
forms and are not tabulated.

\begin{table*}[!htbp]
\small\centering
\caption{\textbf{Standard facts, used without proof.} These are the
only external ingredients in Apps.~\ref{app:tractability}--\ref{app:thmC}.}
\label{tab:standard}
\begin{tabularx}{\textwidth}{@{}c>{\raggedright\arraybackslash}X>{\raggedright\arraybackslash}p{0.245\textwidth}@{}}
\toprule
& Fact & Used in \\
\midrule
(S1) &
\emph{Gaussian conditioning.} If $(X,Y)$ is jointly Gaussian with
$\Cov(X,Y)=C$ and $\Var(Y)=\Sigma_Y\succ0$, then
$X\mid Y{=}y\sim\N\big(\E X+C\Sigma_Y^{-1}(y-\E Y),\;
\Var(X)-C\Sigma_Y^{-1}C^{\top}\big)$. If $\Sigma_Y$ is singular the
same formula holds on $\ran\Sigma_Y$ with the Moore--Penrose
pseudoinverse in place of $\Sigma_Y^{-1}$
\cite[Theorem~2.5.1]{anderson2003introduction}. &
Lemmas~\ref{lem:pinned}--\ref{lem:reverse},
Cor.~\ref{cor:drift} \\[2pt]
(S2) &
\emph{Simultaneous diagonalization.} A family of real symmetric
matrices is simultaneously orthogonally diagonalizable if and only if
its members commute pairwise \cite[Ch.~4.5]{horn2012matrix}. &
Theorem~\ref{thm:tractability}, step (1)$\Rightarrow$(2) \\[2pt]
(S3) &
\emph{Linear SDEs; two-sided Markov property.} A linear SDE with
deterministic coefficients is solved by variation of constants and has
Gaussian marginals with deterministic covariance; and for a Markov
process, conditionally on $X_t$ the past $\sigma(X_s:s\le t)$ is
independent of the future $\sigma(X_u:u\ge t)$
\cite{karatzas1991brownian,oksendal2010stochastic}. &
Lemma~\ref{lem:reverse}, Cor.~\ref{cor:drift},
Lemma~\ref{lem:deficit-rec} \\[2pt]
(S4) &
\emph{Riemann-sum error via total variation.} For $f$ of bounded
variation on $[c,d]$ and any partition, the sum of the oscillations of
$f$ over the subintervals is at most $\TV_{[c,d]}(f)$; for $f$
monotone on each of two subintervals, $\TV(f)$ is bounded by the sum
of its endpoint values \cite[Ch.~6]{rudin1976principles}. &
Lemmas~\ref{lem:riemann}, \ref{lem:H-deriv} \\[2pt]
(S5) &
\emph{Descartes' rule of signs.} The number of positive real roots of
a nonzero real polynomial, counted with multiplicity, is at most the
number of sign changes in its coefficient sequence, and differs from it
by an even number; see \cite[Ch.~2]{basu2006algorithms} for its use in
exact certificates. &
the exact integer uniqueness certificates for $2\le T\le120$,
App~\ref{app:uniqueness} \\[2pt]
(S6) &
\emph{Gaussian KL.} For $p=\N(\mu_1,\sigma_1^2)$ and
$q=\N(\mu_2,\sigma_2^2)$,
$\KL(p\,\|\,q)=\tfrac12(u-1-\ln u)+\tfrac{(\mu_1-\mu_2)^2}{2\sigma_2^2}$
with $u=\sigma_1^2/\sigma_2^2$ \cite[Ch.~8]{cover2012elements}. &
Theorem~\ref{thm:invisibility}, Prop.~\ref{prop:constant},
Theorem~\ref{thm:colorblind}, and the definition of $\Phi$ throughout
App~\ref{app:thmB} \\
\bottomrule
\end{tabularx}
\end{table*}

\subsection{Proofs for Section~\ref{sec:tractability} (Tractability)}
\label{app:tractability}
\subsubsection{The general Gaussian bridge}
\label{app:bridge}

We first show that \emph{every} deterministic matrix-valued covariance
gives a tractable bridge at the multivariate level.
Theorem~\ref{thm:tractability} then identifies exactly when that
tractability separates into scalar modes.

\begin{lemma}[Exact pinned marginal]
\label{lem:pinned}
For every $t\in(0,1)$,
\begin{equation}
X_t\mid(X_0{=}x_0,\,X_1{=}x_1)\;\sim\;\N\bigl(\mu_t,\,\Gamma_t\bigr),
\end{equation}
where
\begin{align}
\mu_t    &:= x_0+A(t)V^{-1}(x_1-x_0),\label{eq:pinned-mean}\\
\Gamma_t &:= A(t)-A(t)V^{-1}A(t).\label{eq:pinned-var}
\end{align}
\end{lemma}

\begin{proof}
Condition on $X_0=x_0$. Then
\begin{equation}
X_t=x_0+\int_0^tL\dd B,
\qquad
X_1=x_0+\int_0^1L\dd B,
\end{equation}
so $\Var(X_t\mid X_0)=A(t)$ and $\Var(X_1\mid X_0)=V$. Independence of
Brownian increments gives $\Cov(X_t,X_1\mid X_0)=A(t)$. Applying
Gaussian conditioning (S1) to the jointly Gaussian pair $(X_t,X_1)$
proves the claim.
\end{proof}

\begin{lemma}[Exact reverse sub-bridge kernel]
\label{lem:reverse}
For $0<s<t\le1$ with $A(t)\succ0$,
\begin{equation}
X_s\mid(X_t{=}x_t,\,X_0{=}x_0)\;\sim\;
\N\bigl(\mu_{s\mid t},\,\Gamma_{s\mid t}\bigr),
\end{equation}
where
\begin{align}
\mu_{s\mid t}    &:= x_0+A(s)A(t)^{-1}(x_t-x_0),\\
\Gamma_{s\mid t} &:= A(s)-A(s)A(t)^{-1}A(s),
\end{align}
and moreover
$\mathcal{L}(X_s\mid X_t,X_0,X_1)=\mathcal{L}(X_s\mid X_t,X_0)$.
\end{lemma}

\begin{proof}
Condition on $X_0$. The pair $(X_s,X_t)$ is jointly Gaussian with
$\Var(X_s\mid X_0)=A(s)$, $\Var(X_t\mid X_0)=A(t)$ and
$\Cov(X_s,X_t\mid X_0)=A(s)$, so Gaussian conditioning (S1) gives the
stated kernel. The final equality is the two-sided Markov property
(S3): after conditioning on $X_t$, the past is independent of the
future endpoint.

If $A(t)$ is singular, the directions in $\ker A(t)$ have accumulated
no noise by time $t$ and remain deterministically equal to $x_0$. Apply
the kernel on $\ran A(t)$, or equivalently replace $A(t)^{-1}$ by the
pseudoinverse. In the modal setting this is the convention for modes
with $a_k(t)=0$. On the sampler's grids only $\rho_0=0$ is degenerate,
and no reverse step is taken there.
\end{proof}

\subsubsection{Proof of Theorem~\ref{thm:tractability}}
\label{app:tract-thm}

\begin{proof}
\textbf{(1)$\Rightarrow$(2).} Each $Q(t)$ is real symmetric, hence
orthogonally diagonalizable. The linear span of $\{Q(t):t\notin N\}$ is
a finite-dimensional subspace of the symmetric matrices, so finitely
many members $Q(t_1),\dots,Q(t_m)$ of the family span it. By~(1) these
commute pairwise, and a finite commuting family of real symmetric
matrices is simultaneously orthogonally diagonalizable (S2):
diagonalize one matrix, observe that every other matrix preserves its
eigenspaces, and recurse on the restrictions. The resulting $U$
diagonalizes $Q(t_1),\dots,Q(t_m)$, hence every matrix in their span,
hence every $Q(t)$ outside $N$. Redefining $Q$ on a null set does not
change the law of the SDE. Finally $Q\succeq0$ forces all diagonal
entries to be nonnegative.

\textbf{(2)$\Rightarrow$(3).} In the coordinates $Y=U^\top X$ the
increment covariance is
$U^\top Q(t)U\dd t=\diag(q_k(t))\dd t$, so we may choose a
deterministic diagonal square root. The coordinates are jointly
Gaussian with zero cross-covariance at all times, hence independent
scalar Gaussian processes, and each is a Brownian motion under the
clock $a_k(t)$.

\textbf{(3)$\Rightarrow$(1).} Independent scalar modes in a fixed basis
have diagonal instantaneous covariance in that basis, so all $Q(t)$
share the diagonalizer $U$; diagonal matrices commute.

\textbf{The pinned formulas.} Substitute the diagonalizations
$A(t)=U\diag(a_k(t))U^\top$ and $V=U\diag(v_k)U^\top$
into Lemma~\ref{lem:pinned}. In mode $k$ the mean is
\begin{equation}
y_{k,0}+\frac{a_k(t)}{v_k}\bigl(y_{k,1}-y_{k,0}\bigr)
=(1-\rho_k)y_{k,0}+\rho_ky_{k,1},
\end{equation}
and the variance is
\begin{equation}
a_k(t)-\frac{a_k(t)^2}{v_k}=v_k\,\rho_k(1-\rho_k).
\end{equation}
The conditional covariance is diagonal, so the modes remain
independent. Substituting into Lemma~\ref{lem:reverse} gives the
reverse mean coefficient $a_k(s)/a_k(t)=r_k(s,t)$ and the variance
$a_k(s)\bigl(1-r_k(s,t)\bigr)$.
\end{proof}

\subsubsection{Proof of Corollary~\ref{cor:bijection}}
\label{app:tract-bij}

\begin{proof}
Absolute continuity gives $a_k(t)=\int_0^tq_k=v_k\rho_k(t)$.
Nonnegativity of $q_k$ is equivalent to monotonicity of $\rho_k$, and
the endpoint conditions give $a_k(0)=0$ and $a_k(1)=v_k$. Conversely,
define $\rho_k:=a_k/v_k$.
\end{proof}

\paragraph{Four consequences stated only here.}
\begin{corollary}[Static color is a strict subfamily]
\label{cor:prism}
Suppose $Q(t)=\beta(t)\Sigma$ with $\Sigma=U\diag(v_k)U^\top$ and
$\int_0^1\beta=1$. Then Theorem~\ref{thm:tractability} applies with
the \emph{shared} schedule
$\rho_k\equiv\rho(t)=\int_0^t\beta$ and mode-dependent totals $v_k$.
White $I^2$SB is the special case $\Sigma=cI$. A fixed color therefore
allows different $v_k$ but forces every mode to use the same $\rho$.
\end{corollary}

\begin{corollary}[Aliasing blocks]
\label{cor:blocks}
Suppose $\R^d=E_1\oplus\dots\oplus E_m$ is a fixed orthogonal
decomposition and every $Q(t)$ is block diagonal with respect to it.
Then Lemmas~\ref{lem:pinned}--\ref{lem:reverse} decompose the bridge
into independent finite-dimensional blocks, \emph{even when the
matrices inside a block do not commute}; each block uses the matrix
formulas in $A(t)$ and $V$. This is the correct treatment of
downsampling aliasing: a frequency and its foldovers form one coupled
block, and small matrix-valued block schedules replace scalar
schedules.
\end{corollary}

\begin{proof}
Let $\Pi_j$ be the orthogonal projection onto $E_j$. Since every $Q(t)$
is block diagonal, so are $A(t)=\int_0^tQ$ and $V=A(1)$. Writing
$X^{(j)}_t:=\Pi_jX_t$, for $j\ne l$ the increments satisfy
\begin{equation}
\Cov\bigl(X^{(j)}_t-X^{(j)}_s,\;X^{(l)}_t-X^{(l)}_s\bigr)
=\Pi_j\bigl(A(t)-A(s)\bigr)\Pi_l^\top=0,
\end{equation}
so the jointly Gaussian block processes $(X^{(j)})_j$ are mutually
independent. Consequently the pinned marginal and the reverse kernel
factor over blocks, with the matrix formulas applied blockwise to
$A^{(j)}(t):=\Pi_jA(t)\Pi_j^\top$ and $V^{(j)}:=\Pi_jV\Pi_j^\top$. No
commutativity within a block is used.
\end{proof}

\begin{corollary}[Simultaneously diagonal linear drift and diffusion]
\label{cor:drift}
Consider $\dd X_t=F(t)X_t\dd t+L(t)\dd B_t$ and suppose a single
orthogonal $U$ diagonalizes both coefficients,
$U^\top F(t)U=\diag(f_k(t))$ and $U^\top Q(t)U=\diag(q_k(t))$ a.e.
Put
\begin{equation}
g_k(t,s):=\exp\!\int_s^tf_k,
\qquad
\sigma^2_k(t):=\int_0^tg_k(t,u)^2q_k(u)\dd u.
\end{equation}
Then the reference and all pinned bridges decompose exactly into scalar
modes, with
\begin{align}
\E[Y_{k,t}\mid y_{k,0},y_{k,1}]
&=g_k(t,0)\,y_{k,0}\notag\\
&\quad+\frac{\sigma^2_k(t)\,g_k(1,t)}{\sigma^2_k(1)}
  \bigl(y_{k,1}-g_k(1,0)y_{k,0}\bigr),\\
\Var(Y_{k,t}\mid y_{k,0},y_{k,1})
&=\sigma^2_k(t)\notag\\
&\quad-\frac{\sigma^2_k(t)^2\,g_k(1,t)^2}{\sigma^2_k(1)},
\end{align}
and for $0<s<t$ the reverse kernel has mean
\begin{equation}
g_k(s,0)y_{k,0}
+\frac{\sigma^2_k(s)\,g_k(t,s)}{\sigma^2_k(t)}
 \bigl(y_{k,t}-g_k(t,0)y_{k,0}\bigr)
\end{equation}
and variance
$\sigma^2_k(s)-\sigma^2_k(s)^2g_k(t,s)^2/\sigma^2_k(t)$.
\end{corollary}

\begin{proof}
Variation of constants (S3) gives
$Y_{k,t}=g_k(t,0)y_{k,0}+\int_0^tg_k(t,u)\sqrt{q_k(u)}\dd B_{k,u}$,
whence
\begin{align}
\Var(Y_{k,t}\mid y_{k,0})&=\sigma^2_k(t),\\
\Cov(Y_{k,t},Y_{k,1}\mid y_{k,0})&=\sigma^2_k(t)\,g_k(1,t),\\
\Cov(Y_{k,s},Y_{k,t}\mid y_{k,0})&=\sigma^2_k(s)\,g_k(t,s).
\end{align}
Gaussian conditioning (S1) gives the stated formulas, and independence
across modes follows from the common diagonalization.
\end{proof}

\begin{remark}[Hilbert space]
\label{rem:hilbert}
Assume $\Hil$ is separable, the covariance operators share an
orthonormal eigenbasis with $Q(t)e_k=q_k(t)e_k$, and
$\sum_kv_k<\infty$ (trace class). Then the reference is an
$\Hil$-valued Gaussian process and Theorem~\ref{thm:tractability}
applies coordinatewise, the Gaussian series converging in
$L^2(\Omega;\Hil)$. For cylindrical noise the modal identities remain
formal in a larger distribution space. For discrete images the
finite-dimensional theorem suffices. Complex Fourier coefficients are
handled either by a unitary/Hermitian version or by grouping conjugate
pairs into two-dimensional real blocks
(Corollary~\ref{cor:blocks}).
\end{remark}

\subsection{Proofs for Section~\ref{sec:invisibility} (Invisibility)}
\label{app:thm0}

\subsubsection{The sampler as an affine recursion}
\label{app:affine}

Write $\bar c_i:=(1-\rho_i)W+\rho_i$. The plug-in step from level
$\rho_i$ to $\rho_{i-1}$ is
\begin{align}
\hat x_0 &= Wx_1+K(\rho_i)\bigl(x_{\rho_i}-\bar c_ix_1\bigr),\\
x_{\rho_{i-1}} &= \hat x_0+r_i\bigl(x_{\rho_i}-\hat x_0\bigr)+\sqrt{q_i}\,\xi_i,
\end{align}
with $r_i=\rho_{i-1}/\rho_i$ and $q_i=v\rho_{i-1}(1-r_i)$. Every
operation is affine in $(x,x_1)$ and adds independent Gaussian noise.
Conditional on $x_1$ the state is therefore \emph{exactly}
$\N(c_ix_1,V_i)$, where
\begin{align}
A_i     &= r_i+(1-r_i)K(\rho_i),\label{eq:Ai}\\
B_i     &= (1-r_i)\bigl(W-K(\rho_i)\bar c_i\bigr),\label{eq:Bi}\\
c_{i-1} &= A_ic_i+B_i, &&c_T=1,\label{eq:crec}\\
V_{i-1} &= A_i^2V_i+q_i, &&V_T=0.\label{eq:Vrec}
\end{align}

\begin{lemma}[Mean exactness at every $T$]
\label{lem:mean-exact}
$c_i=(1-\rho_i)W+\rho_i$ for all $i$, for any grid, any $T\ge1$ and any
$v\ge0$. In particular $c_0=W$ exactly.
\end{lemma}

\begin{proof}
Backward induction. Base case: $c_T=1=(1-\rho_T)W+\rho_T$. Step:
substitute $c_i=\bar c_i$ into \eqref{eq:crec}. The two terms
containing $K(\rho_i)\bar c_i$ cancel exactly, leaving
\begin{align}
c_{i-1}&=r_i\bar c_i+(1-r_i)W\notag\\
       &=W(1-r_i\rho_i)+r_i\rho_i\notag\\
       &=(1-\rho_{i-1})W+\rho_{i-1},
\end{align}
where we used $r_i\rho_i=\rho_{i-1}$.
\end{proof}

\begin{lemma}[Exact telescoping]
\label{lem:telescope}
$A_i=\varphi(\rho_{i-1})/\varphi(\rho_i)$, and hence
\begin{equation}
\prod_{m=j+1}^{i}A_m=\frac{\varphi(\rho_j)}{\varphi(\rho_i)},
\qquad
\prod_{m=1}^{i}A_m=\frac{P}{\varphi(\rho_i)}.
\label{eq:telescope}
\end{equation}
\end{lemma}

\begin{proof}
By \eqref{eq:Ai},
$A_i=\bigl[r_i\varphi(\rho_i)+(1-r_i)P\bigr]/\varphi(\rho_i)$, and the
numerator telescopes:
\begin{align}
r_i\varphi(\rho_i)+(1-r_i)P
&=r_i\bigl[(1-\rho_i)P+v\rho_i\bigr]+(1-r_i)P\notag\\
&=P(1-r_i\rho_i)+v\,r_i\rho_i\notag\\
&=P(1-\rho_{i-1})+v\rho_{i-1}\notag\\
&=\varphi(\rho_{i-1}).
\end{align}
\end{proof}


\begin{lemma}[One-step identity and deficit recursion]
\label{lem:deficit-rec}
Let $D_i:=\Sigma(\rho_i)-V_i$ be the gap between the exact ancestral
chain and the plug-in chain. Then
\begin{equation}
D_{i-1}=A_i^2D_i+(1-r_i)^2M(\rho_i),
\qquad D_T=0.
\label{eq:deficit-rec}
\end{equation}
Every term is nonnegative and the source term is strictly positive when
$v>0$; this proves Corollary~\ref{cor:underdispersion}.
\end{lemma}

\begin{proof}
By the two-sided Markov property (S3) the exact ancestral chain
reproduces the true conditional law at every level. That chain draws
$x_0\sim\N(\cdot,M(\rho_i))$ rather than replacing the variable by its
conditional mean, so its variance obeys
\begin{equation}
\Sigma(\rho_{i-1})=A_i^2\Sigma(\rho_i)+(1-r_i)^2M(\rho_i)+q_i.
\end{equation}
The plug-in chain obeys the same recursion with the middle term
omitted, namely \eqref{eq:Vrec}. Subtracting gives
\eqref{eq:deficit-rec}. Nonnegativity of $D_0$ is immediate by
backward induction from $D_T=0$.
\end{proof}

\begin{lemma}[Deficit in closed form]
\label{lem:deficit-closed}
For every grid and every $T\ge1$,
\begin{equation}
D_0=vP^3\sum_{i=1}^{T}
\frac{(\Delta\rho_i)^2}{\rho_i\,\varphi(\rho_i)\,\varphi(\rho_{i-1})^2}.
\label{eq:deficit-closed}
\end{equation}
For $P=v$ on the uniform grid this reduces to $D_0=(P/T)\Har_T$ with
$\Har_T$ the $T$-th harmonic number.
\end{lemma}

\begin{proof}
Unroll \eqref{eq:deficit-rec} from $D_T=0$ and use
\eqref{eq:telescope} for the accumulated contractions, together with
the one-step identity
\begin{equation}
(1-r_i)^2M(\rho_i)=\frac{vP\,(\Delta\rho_i)^2}{\rho_i\,\varphi(\rho_i)}.
\end{equation}
Each surviving term acquires the factor
$\bigl(P/\varphi(\rho_{i-1})\bigr)^2$, which produces the stated
$\varphi(\rho_{i-1})^2$ in the denominator and the overall $P^3$. When
$P=v$ we have $\varphi\equiv P$, so the sum collapses to
$\sum_i(\Delta\rho_i)^2/\rho_i=(1/T)\sum_{i\le T}1/i$.
\end{proof}

\subsubsection{Proof of Theorem~\ref{thm:invisibility}}
\label{app:thm0-proof}

\begin{proof}
\textbf{(a) $v>0$.} Put $\varphi_{\min}=\min(P,v)>0$,
$\varphi_{\max}=\max(P,v)$ and let
\begin{equation}
\Theta_T:=\sum_{i=1}^T\frac{(\Delta\rho_i)^2}{\rho_i}
\end{equation}
be the purely geometric part of \eqref{eq:deficit-closed}. Since
$\varphi_{\min}\le\varphi(\rho)\le\varphi_{\max}$ on $[0,1]$,
\begin{equation}
\frac{vP^3}{\varphi_{\max}^3}\,\Theta_T
\;\le\;D_0\;\le\;
\frac{vP^3}{\varphi_{\min}^3}\,\Theta_T .
\label{eq:sandwich}
\end{equation}
It therefore suffices to show $\Theta_T\to0$.

\emph{Uniform grid.}
$\Theta_T=(1/T)\Har_T=(\log T)/T+O(1/T)$.

\emph{Power grid $\rho_i=(i/T)^a$, $a>1$.} The mean value theorem
gives $\Delta\rho_i\le a\,i^{a-1}/T^a$, hence
\begin{equation}
\frac{(\Delta\rho_i)^2}{\rho_i}\le\frac{a^2\,i^{a-2}}{T^a}.
\end{equation}
The summand is \emph{not} uniformly of order $T^{-2}$: the $i=1$ term
is $T^{-a}$. Now $u\mapsto u^{a-2}$ is monotone on $[1,\infty)$
(nonincreasing for $1<a\le2$, nondecreasing for $a\ge2$), so comparing
the sum with the integral term by term --- adding the first term
separately in the nonincreasing case --- gives
\begin{align}
\sum_{i=1}^Ti^{a-2}
&\le1+\int_1^{T+1}u^{a-2}\dd u\notag\\
&\le\Bigl(1+\tfrac{2^{a-1}}{a-1}\Bigr)T^{a-1},
\end{align}
where we used $(T+1)^{a-1}\le(2T)^{a-1}$ and $1\le T^{a-1}$.
Therefore
\begin{equation}
\Theta_T\le a^2\Bigl(1+\tfrac{2^{a-1}}{a-1}\Bigr)\frac1T=O(1/T).
\end{equation}

In both cases $D_0\to0$ by \eqref{eq:sandwich}, so $V_0\to P$.
Lemma~\ref{lem:mean-exact} shows the mean is already exact, so the
terminal law converges to $\N(Wx_1,P)$. By (S6) the terminal KL is
\begin{equation}
\tfrac12\Bigl(\tfrac{V_0}{P}-1-\log\tfrac{V_0}{P}\Bigr)
=\frac{D_0^2}{4P^2}+O\bigl((D_0/P)^3\bigr),
\label{eq:KLexpand}
\end{equation}
and the limiting law depends on neither $v$ nor the schedule.

\textbf{(b) $v=0$.} All $q_i=0$, so $V_i\equiv0$ and the terminal law
is the point mass at $c_0x_1=Wx_1$; Lemma~\ref{lem:mean-exact} remains
valid. The first step from $\rho=1$ is defined directly by
$\hat x_0=Wx_1$, because the scaled deviation is identically zero.
Hence $\KL\bigl(\delta(Wx_1)\,\|\,\N(Wx_1,P)\bigr)=+\infty$. The map
from $v$ to the limiting law is constant on $v>0$ and jumps at $v=0$.
\end{proof}

\subsubsection{Why resource-free objectives are degenerate}
\label{app:deadroute}

The integrated-Bayes-risk objective
\begin{equation}
J=\sum_k\int w(t)\,\MMSE_k(t)\dd t
\end{equation}
is concave and increasing in every $v_k$. Under a budget the matched point does satisfy the Lagrange
condition --- but it is the \emph{maximum}; the minima sit at vertex
allocations. An explicit counterexample is $\rho=0.5$, $P=(1,4)$ and
budget $5$: the matched value is $2.5$, whereas a vertex allocation
gives $0.833$.

Minimizing target Bayes risk rewards a bridge
state that is maximally informative about $x_0$, and without a resource
constraint the optimum is $v=0$. That choice makes the interpolation
deterministic, reduces the method to one-shot Wiener regression, and
destroys distributional correctness because the terminal variance tends
to $0$. An objective whose optimum lies on the $v=0$ boundary is
unsuitable for a \emph{bridge}, whose purpose is distributional. 
\subsection{Proofs for Section~\ref{sec:finitestep} (Finite-Step)}
\label{app:thmB}

\subsubsection{Proof of Theorem~\ref{thm:proportional}}
\label{app:prop-proof}

\begin{proof}
\textbf{Step 1 (scale symmetry).} Write $\varphi(\rho)=P\psi(\rho)$
with $\psi(\rho)=1+(x-1)\rho$ and $x=v/P$. In
\eqref{eq:deficit-closed} the factor $P^3$ cancels the three factors of
$P$ contributed by the $\varphi$ terms, and the remaining factor is
$v=xP$. Thus $x$ survives in the dimensionless expression while $P$
normalizes $D_0$:
\begin{equation}
\frac{D_0}{P}=x\,G(x),
\qquad
G(x)=\sum_i
\frac{(\Delta\rho_i)^2}
     {\rho_i\,\psi(\rho_i)\,\psi(\rho_{i-1})^2}.
\label{eq:HxG}
\end{equation}
Hence $\KL=\Phi(x)$ depends on the reference only through $x$, and
otherwise only on the grid. Structurally, $P$ is the only scale in the
per-mode Gaussian problem and $v/P$ is its only dimensionless
parameter.

\textbf{Step 2 (form of every optimizer).} By Step~1,
$\sum_k\KL_k=\sum_k\Phi(v_k/P_k)$ with the \emph{same} function $\Phi$
in every mode, so the objective is separable. At any \emph{local}
minimizer each coordinate $x_k=v_k/P_k$ must be a local minimizer of
$\Phi$; different coordinates may, however, select different local
minima, so a common mode-independent $x$ is not forced. At any
\emph{global} minimizer $x_k\in\argmin\Phi$ for every $k$. If
$\argmin\Phi$ is the singleton $\{x^{*}\}$ then
$v^{*}_k=x^{*}P_k$ for every mode, with $x^{*}$ mode-independent. If
$\Phi$ has several local minimizers --- which does occur on general
grids, by Proposition~\ref{prop:uniqueness} --- then assigning
different modes to different valleys produces genuinely
non-proportional \emph{local} optimizers of the sum.

\textbf{Step 3 (barriers, hence an interior optimum).} As $x\to0$ the
observation-end term ($i{=}T$) of $xG(x)$ tends to $1$: the sampler
cannot create posterior variance without reference noise, exactly as in
Theorem~\ref{thm:invisibility}(b). Hence $u\to0$ and $\Phi\to\infty$.
As $x\to\infty$ the clean-end term ($i{=}1$) of $xG(x)$ tends to $1$:
excess noise entering the final Bayes step cannot be fully removed, so
again $\Phi\to\infty$. Since $\Phi$ is continuous on $(0,\infty)$ it
attains an interior minimum. Both barriers are visible in the left
panel of Figure~\ref{fig:app-H}.
\end{proof}

\subsubsection{Proof of Proposition~\ref{prop:constant}: the constant $x^{*}(T)$}
\label{app:constant-proof}

Two conventions are in force in this subsection. First, the window
center denoted $x_0$ in the statement of
Proposition~\ref{prop:constant} is the quantity $\chi_T=(2\ln T)^{-1/2}$
here; the symbol $x_0$ is reserved for the clean signal. Second,
\emph{minimizer} means global minimizer of $\Phi$ on $(0,\infty)$;
Proposition~\ref{prop:constant} is proved in this sense.

Throughout this subsection the grid is uniform, $\rho_i=i/T$, and we
write
\begin{align}
H(x)&:=\frac{D_0}{P}=\frac{x\,\Sm(x)}{T},\label{eq:Hdef}\\
\Sm(x)&:=\frac1T\sum_{i=1}^{T}
        \frac{1}{\rho_i\,\psi_i\,\psi_{i-1}^{2}},\label{eq:Sdef}\\
\psi_i&:=\psi(i/T)=1-b\,\tfrac{i}{T},
        \qquad b:=1-x.\label{eq:psii}
\end{align}
We also set
\begin{align}
f(\rho)&:=\frac{1}{\rho\,\psi(\rho)^3},
&\Sm^{\circ}(x)&:=\frac1T\sum_{i=1}^Tf(\rho_i),\\
h(x)&:=x\ln T+\frac{1}{2x},
&\chi_T&:=(2\ln T)^{-1/2},
\end{align}
and $\Wind:=[\chi_T/3,\,3\chi_T]$. The comparison function $h$ is
strictly convex on $(0,\infty)$ with unique minimizer $\chi_T$, minimum
value $h(\chi_T)=\sqrt{2\ln T}$ and $h''(x)=x^{-3}$. Finally,
$\Phi(x)=\tfrac12(u-1-\ln u)$ with $u=1-H(x)\in(0,1)$ is strictly
decreasing in $u$, so \emph{minimizing $\Phi$ is equivalent to
minimizing $H$}. We work throughout with the scaled quantity $T\,H$.

The proof compares $T\,H$ with $h$ in four steps: replace the shifted
grid factor (Lemma~\ref{lem:gridshift}), replace the sum by an integral
(Lemma~\ref{lem:riemann}), evaluate that integral in closed form
(Lemma~\ref{lem:integral}), and combine
(Lemma~\ref{lem:expansion}). Lemma~\ref{lem:global} then confines the
minimizer to $\Wind$, where the comparison is uniform.

\begin{figure*}[!htb]
\centering
\includegraphics[width=\textwidth]{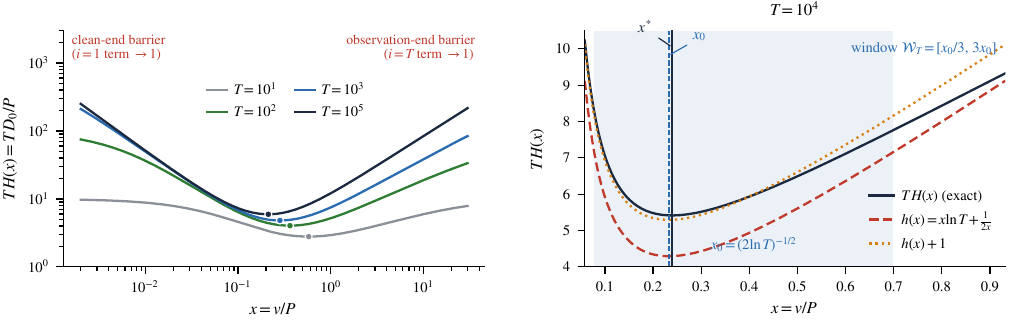}
\caption{\textbf{The landscape of $T\,H(x)$, and why the minimizer is
localized.} \emph{Left:} the two barriers of
Theorem~\ref{thm:proportional}, Step~3 --- $\Phi\to\infty$ at both ends
--- with the exact minimizer marked for each $T$. \emph{Right:} at
$T=10^4$, the exact $T\,H$ against the comparison function
$h(x)=x\ln T+\tfrac1{2x}$ of Lemma~\ref{lem:expansion}; the additive
gap is $1+o(1)$, the shaded band is the window
$\Wind=[\chi_T/3,3\chi_T]$ of Lemma~\ref{lem:global}, and $x^{*}$ has
already almost reached $\chi_T=(2\ln T)^{-1/2}$. Computed from
\eqref{eq:deficit-closed}; no simulation.}
\label{fig:app-H}
\end{figure*}

\begin{lemma}[Grid-shift replacement]
\label{lem:gridshift}
For $0<x<1$ and $T\ge2$,
\begin{equation}
\Bigl(1+\tfrac{1}{Tx}\Bigr)^{-2}\Sm^{\circ}(x)
\;\le\;\Sm(x)\;\le\;\Sm^{\circ}(x).
\end{equation}
\end{lemma}

\begin{proof}
$\psi_{i-1}=\psi_i+b/T$ with $0<b<1$, and $\psi_i\ge\psi_T=x$, so
$\psi_i\le\psi_{i-1}\le\psi_i\bigl(1+\tfrac{1}{Tx}\bigr)$. Square and
invert termwise.
\end{proof}

\begin{lemma}[Riemann comparison]
\label{lem:riemann}
For $0<x\le\tfrac34$ and $T\ge8$,
\begin{equation}
\Bigl|\Sm^{\circ}(x)-1-\int_{1/T}^{1}f(\rho)\dd\rho\Bigr|
\;\le\;1+\frac{12}{T}+\frac{1}{Tx^{3}}.
\end{equation}
\end{lemma}

\begin{proof}
The logarithmic derivative
$(\ln f)'(\rho)=-1/\rho+3b/\psi(\rho)$ vanishes only at
$\rho^{*}=1/(4b)\le1$, using $b\ge\tfrac14$. Hence $f$ strictly
decreases on $(0,\rho^{*}]$ and strictly increases on $[\rho^{*},1]$.

Separate the initial term first. For $T\ge8$,
\begin{equation}
\tfrac1Tf(1/T)=\psi(1/T)^{-3}\in\bigl[1,\,1+\tfrac6T\bigr],
\end{equation}
because $\psi(1/T)\ge1-1/T$ and $(1-u)^{-3}\le1+6u$ on
$[0,\tfrac18]$. For each remaining interval,
$\bigl|\tfrac1Tf(\rho_i)-\int_{\rho_{i-1}}^{\rho_i}f\bigr|
\le\tfrac1T\osc_{[\rho_{i-1},\rho_i]}f$, and by (S4) the sum of those
oscillations is at most $\TV(f)$. Because $f$ first decreases and then
increases, $\TV(f)\le f(1/T)+f(1)$. Therefore
\begin{align}
\Bigl|\sum_{i=2}^{T}\tfrac1Tf(\rho_i)-\int_{1/T}^{1}f\Bigr|
&\le\tfrac1T\bigl(f(1/T)+f(1)\bigr)\notag\\
&\le\psi(1/T)^{-3}+\tfrac{1}{Tx^{3}}\notag\\
&\le1+\tfrac{6}{T}+\tfrac{1}{Tx^{3}}.
\end{align}
Combining the two displays gives the claim.
\end{proof}

\begin{lemma}[The integral in closed form]
\label{lem:integral}
For $0<x<1$ and $T\ge3$,
\begin{equation}
x\int_{1/T}^{1}f(\rho)\dd\rho
=x\ln T+\frac{1}{2x}+1+x\ln\frac1x-\frac{3x}{2}+\theta,
\end{equation}
with $|\theta|\le 6x/T$.
\end{lemma}

\begin{proof}
Clearing denominators in $(1-u)(1+u+u^2)=1-u^3$ with $u=1-b\rho$ gives
the partial-fraction identity
\begin{equation}
\frac{1}{\rho\psi^3}
=\frac1\rho+\frac{b}{\psi}+\frac{b}{\psi^2}+\frac{b}{\psi^3},
\end{equation}
whose antiderivative is
\begin{equation}
F(\rho)=\ln\frac{\rho}{\psi(\rho)}+\frac{1}{\psi(\rho)}
        +\frac{1}{2\psi(\rho)^2},
\qquad F'=f.
\end{equation}
Hence $\int_{1/T}^1f=F(1)-F(1/T)$ with
$F(1)=\ln\tfrac1x+\tfrac1x+\tfrac{1}{2x^2}$. Also
$\psi(1/T)=1-b/T\in[1-\tfrac1T,1]$, so
\begin{align}
F(1/T)&=-\ln T-\ln\psi(1/T)+\frac{1}{\psi(1/T)}
        +\frac{1}{2\psi(1/T)^2}\notag\\
      &=-\ln T+\tfrac32+O\bigl(\tfrac1T\bigr).
\end{align}
For $T\ge3$ that remainder is at most $6/T$ in absolute value.
Multiplying by $x$ completes the proof. (At $T=2$ the constant $6$
fails narrowly as $x\to0$; this is harmless, because
Lemma~\ref{lem:expansion} is used only for large $T$ and
Proposition~\ref{prop:uniqueness} handles $T=2$ exactly.)
\end{proof}

\begin{lemma}[Uniform expansion on the window]
\label{lem:expansion}
Let $\delta_T:=\sup_{x\in\Wind}\bigl|T\,H(x)-h(x)-1\bigr|$. Then
$\delta_T\to0$; explicitly
\begin{equation}
\delta_T=O\Bigl(\chi_T\ln\tfrac1{\chi_T}+\tfrac{\ln T}{T}\Bigr).
\end{equation}
\end{lemma}

\begin{proof}
Take $T$ large enough that $3\chi_T\le\tfrac34$, so that
Lemmas~\ref{lem:gridshift}--\ref{lem:integral} all apply on $\Wind$.
Chaining the three lemmas,
\begin{align}
T\,H(x)
&=x\,\Sm(x)=x\,\Sm^{\circ}(x)+E_1\notag\\
&=x+x\!\int_{1/T}^{1}\!f+xE_2+E_1\notag\\
&=h(x)+1+r(T,x),
\label{eq:chain}
\end{align}
where the remainder collects the six small terms,
\begin{equation}
r(T,x):=x+x\ln\tfrac1x-\tfrac{3x}{2}+\theta+xE_2+E_1 .
\end{equation}
Here $|E_1|\le\tfrac{2}{Tx}\cdot x\Sm^{\circ}=\tfrac{2\Sm^{\circ}}{T}$
by Lemma~\ref{lem:gridshift},
$|E_2|\le1+\tfrac{12}{T}+\tfrac{1}{Tx^{3}}$ by
Lemma~\ref{lem:riemann}, and $|\theta|\le\tfrac{6x}{T}$ by
Lemma~\ref{lem:integral}.

We check each contribution uniformly on $\Wind$. First
$x\le3\chi_T\to0$ and
$x\ln\tfrac1x\le3\chi_T\ln\tfrac{3}{\chi_T}\to0$. Next
$x\cdot\tfrac{1}{Tx^{3}}=\tfrac{1}{Tx^{2}}\le\tfrac{18\ln T}{T}\to0$.
Finally Lemmas~\ref{lem:riemann}--\ref{lem:integral} together with
$\tfrac{1}{2x^{2}}\le9\ln T$ give $\Sm^{\circ}=O(\ln T)$ on the
window, so $\tfrac{2\Sm^{\circ}}{T}=O(\tfrac{\ln T}{T})\to0$. Every
contribution to $r$ therefore vanishes uniformly on $\Wind$.
\end{proof}

\begin{lemma}[No minimizer outside the window]
\label{lem:global}
There is a $T_0$ such that for every $T\ge T_0$ every global minimizer
of $H$ on $(0,\infty)$ lies in $\Wind=[\chi_T/3,\,3\chi_T]$.
\end{lemma}

\begin{proof}
By Lemma~\ref{lem:expansion},
$\min_{x\in\Wind}T\,H(x)\le\Lambda_T$ where
\begin{equation}
\Lambda_T:=\sqrt{2\ln T}+1+\delta_T .
\end{equation}
We show $T\,H(x)>\Lambda_T$ outside $\Wind$ for all large $T$.

\emph{Two lower bounds.} For $0<x\le\tfrac12$ and $T\ge8$, keep only
the terms with $i>T/2$. Since $\rho_i\le1$, $\psi_i\le\psi_{i-1}$ and
$\psi^{-3}$ is increasing there,
\begin{align}
T\,H(x)
&\ge x\sum_{i>T/2}\frac{1/T}{\psi_{i-1}^{3}}
 \;\ge\;x\int_{1/2}^{1-1/T}\frac{\dd\rho}{\psi(\rho)^3}\notag\\
&=\frac{x}{2b}
  \left[\frac{1}{\psi(1-\tfrac1T)^2}-\frac{1}{\psi(\tfrac12)^2}\right]
  \notag\\
&\ge\frac{x}{2}
  \left[\frac{1}{(x+\tfrac1T)^2}-4\right],
\label{eq:tailbound}
\end{align}
where we used $b\le1$,
$\psi(1-\tfrac1T)=x+\tfrac bT\le x+\tfrac1T$ and
$\psi(\tfrac12)\ge\tfrac12$. Separately, for $0<x\le1$ the $i=T$ term
alone gives
\begin{equation}
T\,H(x)\ge\frac{1}{T\bigl(x+\tfrac1T\bigr)^2}.
\label{eq:last-term-bound}
\end{equation}

\emph{The six regions.} Table~\ref{tab:regions} lists them; empty
regions may be ignored. Region~(i) uses
\eqref{eq:last-term-bound} with $x+\tfrac1T\le\tfrac2T$. Regions
(ii)--(iii) use \eqref{eq:tailbound} with, respectively,
$x+T^{-1}\le2x$ and $x+T^{-1}\le x(1+T^{-1/2})$. Region~(iv) uses
$\psi\le1$ on $[0,1]$, so $T\,H\ge x\sum_i1/i\ge x\ln T$. Region~(v)
uses $\psi_{i-1}\le\psi_i\le\tfrac32$ for $i\le T/(2x)$, giving
$T\,H\ge\tfrac{8}{27}x\ln\tfrac{T}{2x}$ and hence
$\tfrac{8}{81}\ln T$ for $T\ge64$. Region~(vi) uses the $i=1$ term,
$T\,H\ge x/\psi(1/T)\ge Tx/(T+x)$.

Comparing the last column of Table~\ref{tab:regions} with
$\Lambda_T=\sqrt{2\ln T}(1+o(1))$: in region~(iii) the leading
coefficient is $\tfrac32\sqrt2>\sqrt2$, in region~(iv) it is
$\tfrac{3}{\sqrt2}>\sqrt2$, and in the remaining regions $T$,
$\sqrt T$ or $\ln T$ grows faster than $\sqrt{\ln T}$. Hence no global
minimizer lies outside $\Wind$ once $T$ is large enough.
\end{proof}

\begin{center}
\small
\captionof{table}{\textbf{The six regions of Lemma~\ref{lem:global}.}
Each row gives a lower bound on $T\,H$, obtained from the tool named
in the proof, that already exceeds the benchmark $\Lambda_T=\sqrt{2\ln T}(1+o(1))$ attained inside the window.
Only regions (iii) and (iv) are decided by a constant, and in both the
constant is $>\sqrt2$.}
\label{tab:regions}
\begin{tabular}{@{}cll@{}}
\toprule
& region & lower bound on $T\,H$ \\
\midrule
(i)   & $x\le T^{-1}$              & $T/4$ \\
(ii)  & $T^{-1}\le x\le T^{-1/2}$  & $\sqrt T/8-2$ \\
(iii) & $T^{-1/2}\le x\le\chi_T/3$ & $\tfrac32\sqrt{2\ln T}(1-2T^{-1/2})-1$ \\
(iv)  & $3\chi_T\le x\le1$         & $\tfrac{3}{\sqrt2}\sqrt{\ln T}$ \\
(v)   & $1\le x\le\sqrt T$         & $\tfrac{8}{81}\ln T$ \\
(vi)  & $x\ge\sqrt T$              & $\sqrt T/2$ \\
\bottomrule
\end{tabular}
\end{center}

\begin{proof}[Proof of Proposition~\ref{prop:constant}]
Theorem~\ref{thm:proportional} guarantees a global minimizer $x^{*}$,
and for $T\ge T_0$ Lemma~\ref{lem:global} places it in $\Wind$. On
$\Wind$ we have $h''\ge(3\chi_T)^{-3}$, hence the quadratic lower bound
\begin{equation}
h(x)\ge h(\chi_T)+\frac{(x-\chi_T)^2}{54\,\chi_T^{3}} .
\end{equation}
Lemma~\ref{lem:expansion} gives both
$T\,H(x^{*})\le T\,H(\chi_T)\le h(\chi_T)+1+\delta_T$ and
$T\,H(x^{*})\ge h(x^{*})+1-\delta_T$. Subtracting,
\begin{equation}
\frac{(x^{*}-\chi_T)^2}{54\,\chi_T^{3}}\le2\delta_T
\quad\Longrightarrow\quad
\Bigl|\frac{x^{*}}{\chi_T}-1\Bigr|
\le\sqrt{108\,\delta_T\,\chi_T}\longrightarrow0 .
\end{equation}
Therefore $x^{*}(T)=(2\ln T)^{-1/2}(1+o(1))$, which is the claim.

For the optimal value,
$H(x^{*})=\tfrac1T\bigl(\sqrt{2\ln T}+1+O(\delta_T)\bigr)$, and
expanding $\Phi=\tfrac12(u-1-\ln u)$ at $u=1-H$ as in
\eqref{eq:KLexpand} gives
\begin{equation}
\Phi(x^{*})=\frac{H(x^{*})^2}{4}\bigl(1+O(H(x^{*}))\bigr)
=\frac{\ln T}{2T^{2}}
 \Bigl(1+O\bigl(\tfrac{1}{\sqrt{\ln T}}\bigr)\Bigr).
\end{equation}
\end{proof}

\subsubsection{Numerical certificates}
\label{app:certificates}

\begin{remark}[Numerical certificate]
\label{rem:certificate}
An independent certificate, using exact derivatives and no finite
differences, supports the finite-$T$ expansion in three ways.
\emph{(i)} On $\Wind$, $\sup|T\,H(x)-h(x)|$ stabilizes near $1.146$ for
$T=10^3$--$10^5$, consistent with the limiting constant $1$ of
Lemma~\ref{lem:expansion} plus the slowly decaying terms
$x\ln\tfrac1x-\tfrac{3x}{2}$. \emph{(ii)} The derivative certificate
$\sup|T\,H'(x)-h'(x)|/\ln T$ decreases as
$0.648$, $0.246$, $0.170$ and $0.126$ over $T=10^2$--$10^5$.
\emph{(iii)} The measured argmin tracks $\chi_T$: $0.211$ against
$0.208$ at $T=10^5$. Table~\ref{tab:xstar} gives the full sweep and
Figure~\ref{fig:app-xstar} plots it.
\end{remark}

\begin{remark}
\label{rem:refuted}
Our initial asymptotic ansatz was $x^{*}\asymp(2/(T\ln T))^{1/3}$,
which balanced the logarithmic bulk against the \emph{discrete}
endpoint term. Exact numerics refuted it decisively: the ratio drifted
from $1.1$ to $17.6$ over $T=5$--$10^5$. Re-examining the continuum
integral revealed the $\tfrac{1}{2x}$ term that
Lemma~\ref{lem:integral} makes rigorous. This episode is the reason for
the protocol behind every numerically certified claim in this paper:
derive the prediction first, then check it against the exact discrete
objective.
\end{remark}

\begin{table}[!htb]
\small\centering
\caption{\textbf{The constant $x^{*}(T)$ against
Proposition~\ref{prop:constant}.} Exact minimization of $\Phi$ on the
uniform grid via \eqref{eq:deficit-closed}. The ratio approaches $1$
slowly, at the proved rate $O(\delta_T^{1/2}\chi_T^{1/2})$; the
$T\le50$ rows are the values quoted in
App~\ref{app:gaussian-details}.}
\label{tab:xstar}
\begin{tabular}{@{}rccc@{}}
\toprule
$T$ & $x^{*}(T)$ & $\chi_T=(2\ln T)^{-1/2}$ & ratio \\
\midrule
$5$      & $0.721$ & $0.557$ & $1.294$ \\
$10$     & $0.577$ & $0.466$ & $1.239$ \\
$50$     & $0.404$ & $0.357$ & $1.130$ \\
$200$    & $0.332$ & $0.307$ & $1.081$ \\
$217$    & $0.329$ & $0.305$ & $1.079$ \\
$10^3$   & $0.283$ & $0.269$ & $1.050$ \\
$10^4$   & $0.240$ & $0.233$ & $1.028$ \\
$10^5$   & $0.212$ & $0.208$ & $1.017$ \\
\bottomrule
\end{tabular}
\end{table}

\begin{figure}[!htb]
\centering
\includegraphics[width=\columnwidth]{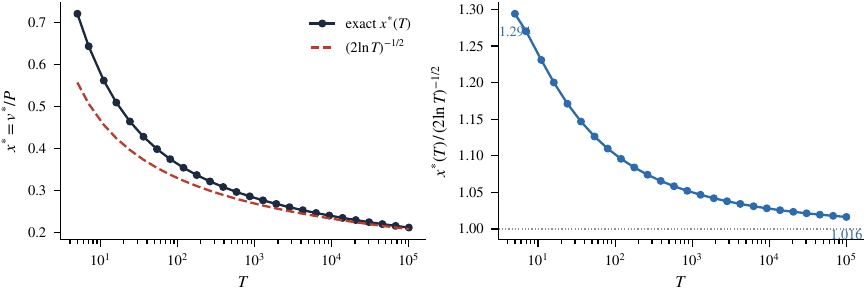}
\caption{\textbf{$x^{*}(T)$ follows $(2\ln T)^{-1/2}$.} \emph{Left:}
exact optimizer against the law of
Proposition~\ref{prop:constant}. \emph{Right:} their ratio, which
decreases from $1.29$ at $T=5$ to $1.017$ at $T=10^5$. The approach is
slow because the proved error is $O(\sqrt{\delta_T\chi_T})$, not
$O(1/T)$.}
\label{fig:app-xstar}
\end{figure}

\subsubsection{Uniqueness (Proposition~\ref{prop:uniqueness})}
\label{app:uniqueness}

\paragraph{Reduction.}
Since $\Phi(x)=\tfrac12(u-1-\ln u)$ is strictly decreasing in
$u\in(0,1)$, minimizing $\Phi$ is equivalent to minimizing
$H(x)=D_0/P$. Put $\omega_i:=(1-\rho_i)/\rho_i$. Then the color path in
the solver coordinate is
\begin{equation}
z_i(x)=\frac{x}{x+\omega_i},
\qquad
H(x)=F\bigl(z(x)\bigr),
\label{eq:colorpath}
\end{equation}
with $F$ as in \eqref{eq:z-deficit}. Uniqueness of $x^{*}$ is thus a
question about how the one-parameter curve $x\mapsto z(x)$ meets the
convex function $F$.

\paragraph{$T=2$: proved.}
Here $H(x)=z_1(x)+\bigl(1-z_1(x)\bigr)^2$. The map $z_1$ is strictly
increasing onto $(0,1)$ and $J(z):=z+(1-z)^2$ is strictly convex, so
$H$ is strictly unimodal and $x^{*}$ is unique.

\paragraph{Uniform grids: proved for $2\le T\le120$.}
For each integer $2\le T\le120$ we derives an explicit
polynomial $P_T\in\mathbb{Z}[x]$ whose positive roots are exactly the
critical points of $H$. The sequence of nonzero coefficients of $P_T$
has exactly one sign change, so by Descartes' rule (S5) $P_T$ has at
most one positive root \cite[Sec.~2.2.1]{basu2006algorithms}. Combined
with the interior-existence argument of
Theorem~\ref{thm:proportional}, Step~3 (the barriers $\Phi\to\infty$ at
both ends), this proves that $H$ has a unique critical point, hence a
unique minimum, for every $2\le T\le120$.

Beyond that certified range, a tolerance-robust numerical scan finds
exactly one local minimum for every $T=2,\dots,200$ and for
$T\in\{500,1000,5000\}$. A general analytic proof remains open. It reduces it either to the conjecture that the relevant
coefficient sequence has one sign change for every $T$, or to verifying
$\Var_p(z)<1/12$ at each critical point, where $1/12$ is the limiting
value in the continuum.

\paragraph{General grids: false, by explicit construction.}
Non-uniqueness already occurs at $T=3$ on a two-cluster grid. Take
$\rho=\bigl(\tfrac{1}{1+C},\tfrac12,1\bigr)$, equivalently
$\omega=(C,1,0)$. Then $z_1=\tfrac{x}{x+C}$, $z_2=\tfrac{x}{x+1}$,
$z_3=1$ and
\begin{equation}
H(x)=z_1+\frac{(z_2-z_1)^2}{z_2}+(1-z_2)^2 .
\label{eq:H3}
\end{equation}
Table~\ref{tab:fivepoints} evaluates \eqref{eq:H3} at five points and
shows that the value $\tfrac34$ is crossed four times, for every
$C\ge100$; By continuity $H$ --- and therefore $\Phi$ --- has at least two local
minima, one in $(\tfrac1{10},\sqrt C)$ and one in $(\sqrt C,10C)$.

For fixed $x$,
$H(x)\to J\bigl(\tfrac{x}{x+1}\bigr)$ with $J$ the strictly unimodal
$T{=}2$ profile above; for $x=Cw$ with $w$ fixed,
$H\to J\bigl(\tfrac{w}{w+1}\bigr)$. Each cluster therefore creates one
$T{=}2$ valley, with minima approaching $\tfrac34$ near $x\approx1$ and
$x\approx C$, and the valleys are separated by the cluster-scale ratio.

\begin{center}
\small
\captionof{table}{\textbf{Five evaluations that force two valleys}
(Eq.~\eqref{eq:H3}, any $C\ge100$; the table reads $C=100$). With
$u:=\tfrac{1}{1+C}$ and $\bar u:=\tfrac{1}{1+\sqrt C}\le\tfrac1{11}$.
Each outer bound keeps a single term of \eqref{eq:H3}; the two
interior identities are exact substitutions. All five were also checked
in exact rational arithmetic.}
\label{tab:fivepoints}
\begin{tabular}{@{}lll@{}}
\toprule
point & value or bound & vs.\ $\tfrac34$ \\
\midrule
$x=\tfrac1{10}$ & $H\ge(1-z_2)^2=\tfrac{100}{121}$        & $>\tfrac34$ \\
$x=1$           & $H=\tfrac34-u+2u^2$                     & $<\tfrac34$ \\
$x=\sqrt C$     & $H\ge z_1+(z_2-z_1)^2=1-3\bar u+4\bar u^2$ & $>\tfrac34$ \\
$x=C$           & $H=\dfrac{3-6u+8u^2-4u^3}{4(1-u)}$      & $<\tfrac34$ \\
$x=10C$         & $H\ge z_1=\tfrac{10}{11}$               & $>\tfrac34$ \\
\bottomrule
\end{tabular}
\end{center}

The interior identities follow by direct substitution, using
$z_1(1)=u$, $z_2(1)=\tfrac12$, $z_1(C)=\tfrac12$ and $z_2(C)=1-u$. For
the middle lower bound we drop $(1-z_2)^2\ge0$ and use $z_2\le1$;
moreover $\bar u\le\tfrac1{11}$ lies below $\tfrac{3-\sqrt5}{8}$, the
smaller root of $4\bar u^2-3\bar u+\tfrac14$.

\paragraph{Severity.}
Clustering can push the count higher. A two-cluster grid at $T{=}48$
with nodes concentrated near $10^{-7}$ and near $1$ produces seven
local minima of $H$, each confirmed in 50-digit arithmetic. The
consequence for the theory is the conditional statement in
Theorem~\ref{thm:proportional}: every \emph{local} optimizer is
mode-wise proportional to \emph{some} local minimizer of $\Phi$, and
the common form $v_k^{*}=x^{*}P_k$ holds for global optimizers exactly
when $\argmin\Phi$ is a singleton.

On a two-cluster grid of
this type, a two-mode instance with $P=(1,4)$ has a genuine local
optimizer whose modes occupy different valleys, with $v_2/v_1$ of order
$10^{6}$ instead of the proportional ratio $4$; perturbations in every
direction confirm local optimality.

\subsubsection{Budget bending (Proposition~\ref{prop:budget})}
\label{app:budget}

The relevant asymptotic regime holds the budget $V$ fixed while
$T\to\infty$; this is also the regime the experiments test. The proof
uses Lemmas~\ref{lem:gridshift}--\ref{lem:integral} plus one derivative
lemma of the same kind. Throughout, $x_k:=v_k/P_k$ and
$\bar x_k:=VP_k/\sum_jP_j^2$; the hypothesis implies
$\bar x_{\max}\le\tfrac34$.

\begin{lemma}[Expansion and derivative on extended windows]
\label{lem:H-deriv}
Fix $X\ge1$. Uniformly on $x\in[T^{-1/4},X]$, on the uniform grid,
\begin{align}
T\,H(x)  &= x\ln T+\tfrac{1}{2x}+O(1),\label{eq:Hexp}\\
T\,H'(x) &= \ln T-\tfrac{1}{2x^2}+O\bigl(\tfrac1x\bigr),\label{eq:Hpexp}
\end{align}
with constants depending only on $X$.
\end{lemma}

\begin{proof}
\textbf{Extension of the three lemmas.} The restriction $x\le\tfrac34$
entered only through the decreasing-then-increasing shape of $f$ in
Lemma~\ref{lem:riemann}. For $b<\tfrac14$ (including $b\le0$) one
checks $(\ln f)'=-1/\rho+3b/\psi<0$ on all of $(0,1]$, so $f$ is
monotone and the total-variation bound only improves. The antiderivative
of Lemma~\ref{lem:integral} is valid for every $b\ne0$, and $x=1$ is
the trivial case $\psi\equiv1$. So
Lemmas~\ref{lem:gridshift}--\ref{lem:integral} extend from
$(0,\tfrac34]$ to $(0,X]$.

\textbf{The value.} Chaining as in \eqref{eq:chain} and noting
$x\ln\tfrac1x=O(1)$ and $\tfrac{1}{Tx^{3}}\le T^{-1/4}$ on the window
gives \eqref{eq:Hexp}.

\textbf{The derivative.} From $T\,H=x\Sm(x)$ we get
$T\,H'=\Sm+x\Sm'$, and \eqref{eq:Hexp} gives
$\Sm=\ln T+\tfrac{1}{2x^2}+O(\tfrac1x)$. Differentiating
\eqref{eq:Sdef} termwise with $\partial_x\psi(\rho)=\rho$,
\begin{equation}
-\Sm'(x)=\frac1T\sum_{i=1}^{T}
\left[\frac{1}{\psi_i^{2}\psi_{i-1}^{2}}
     +\frac{2\,\rho_{i-1}/\rho_i}{\psi_i\,\psi_{i-1}^{3}}\right],
\end{equation}
which is a pair of Riemann-type sums of the integrands $\psi^{-4}$,
monotone in $\rho$, up to the grid shift of
Lemma~\ref{lem:gridshift} and the factor
$\rho_{i-1}/\rho_i=1-\tfrac1i$, whose deviation contributes at most
$\tfrac{2\ln T}{Tx^{4}}=O(\tfrac1x)$ on the window. Since
\begin{equation}
\int_0^1\psi^{-4}\dd\rho=\frac{x^{-3}-1}{3b}
=\frac{x^{-3}}{3}+O(x^{-2})
\end{equation}
and the total-variation Riemann error is $O(x^{-4}/T)=O(\tfrac1x)$
there, we get $-\Sm'=x^{-3}+O(x^{-2})$ and hence
$T\,H'=\ln T+\tfrac{1}{2x^{2}}-\tfrac{1}{x^{2}}+O(\tfrac1x)$, which is
\eqref{eq:Hpexp}.
\end{proof}

\begin{proof}[Proof of Proposition~\ref{prop:budget}]
Because $\Phi\to\infty$ at $0$ and $\Phi$ is continuous, a global
optimizer exists in the interior and satisfies the Lagrange condition
$\tfrac{1}{P_k}\Phi'(x_k)=\lambda$ for every $k$, where
$\Phi'(x)=\tfrac12\tfrac{H}{1-H}H'(x)$. In particular
$\operatorname{sign}\Phi'=\operatorname{sign}H'$ and
$\Phi=\tfrac{H^2}{4}(1+O(H))$.

\textbf{Step 1 (value benchmark).} The feasible point
$\bar v_k=\bar x_kP_k$ has objective value
\begin{equation}
\Phi^{\mathrm{bm}}_T:=\sum_k\Phi(\bar x_k)
=\frac{\ln^2T}{4T^2}\Bigl(\sum_k\bar x_k^2\Bigr)\bigl(1+o(1)\bigr)
\end{equation}
by Lemma~\ref{lem:H-deriv}. Any optimizer does at least as well, so
every coordinate satisfies
$\Phi(x_k)\le\Phi^{\mathrm{bm}}_T=O(\ln^2T/T^2)$.

\textbf{Step 2 (localization).} First exclude $x_k\le T^{-1/4}$: there
\eqref{eq:tailbound} gives
$T\,H\ge\tfrac{1}{2x}(1-2T^{-1/2})-2x\ge\tfrac{T^{1/4}}{3}$, so
$\Phi(x_k)\ge\tfrac{H^2}{4}\gg\Phi^{\mathrm{bm}}_T$. Next exclude
$x_k\ge x_{\mathrm{hi}}$ for a large fixed constant
$x_{\mathrm{hi}}$ depending only on $\sum_j\bar x_j^2$, using the
regional bounds of Table~\ref{tab:regions}: $T\,H\ge x\ln T$ on
$[T^{-1/4},1]$, $T\,H\ge\tfrac{8}{27}x\ln\tfrac{T}{2x}$ on
$[1,\sqrt T]$, and $T\,H\ge\tfrac{\sqrt T}{2}$ beyond.

Now $\lambda>0$. If $\lambda\le0$ then $H'(x_k)\le0$ for every $k$, so
on $[T^{-1/4},x_{\mathrm{hi}}]$ Lemma~\ref{lem:H-deriv} would force
$x_k\le(2\ln T)^{-1/2}(1+o(1))$ and hence
$\sum_kv_k\le(2\ln T)^{-1/2}(1+o(1))\sum_kP_k<V$ for large $T$,
contradicting the budget. Therefore $\lambda>0$, every coordinate has
$H'(x_k)>0$, and $x_k\ge(2\ln T)^{-1/2}(1-o(1))$.

Finally all coordinates are bounded below by a \emph{fixed} constant.
The budget forces some
$x_{k_0}\ge V/\sum_jP_j=:x_{\mathrm{lo}}>0$, and
Lemma~\ref{lem:H-deriv} then gives
$\lambda=\tfrac{1}{P_{k_0}}\Phi'(x_{k_0})\ge c\,\tfrac{\ln^2T}{T^2}$.
Suppose some $x_k\le\varepsilon$ with $\varepsilon$ small and fixed.
Then $\Phi'(x_k)=\lambda P_k\ge c'\tfrac{\ln^2T}{T^2}$ would require
\begin{equation}
\Bigl(x_k\ln T+\tfrac{1}{2x_k}\Bigr)
\Bigl(\ln T-\tfrac{1}{2x_k^2}\Bigr)\ge c''\ln^2T ,
\end{equation}
and for $x_k\le\varepsilon$ the first factor can reach $c''\ln T$ only
through $\tfrac{1}{2x_k}\gtrsim\ln T$, which makes the second factor
negative --- a contradiction. Hence every
$x_k\in[c',x_{\mathrm{hi}}]$, a fixed compact interval.

\textbf{Step 3 (stationarity gives the law).} On that compact,
Lemma~\ref{lem:H-deriv} gives uniformly
\begin{align}
\Phi'(x)
&=\tfrac12H(x)H'(x)\bigl(1+O(H)\bigr)\notag\\
&=\frac{x\ln^2T}{2T^2}
  \Bigl(1+O\bigl(\tfrac{1}{\ln T}\bigr)\Bigr),
\end{align}
so $\tfrac{1}{P_k}\Phi'(x_k)=\lambda$ becomes
$x_k=\kappa P_k\bigl(1+O(1/\ln T)\bigr)$ with the common constant
$\kappa=2\lambda T^2/\ln^2T$. The budget equation
$V=\sum_kx_kP_k=\kappa\sum_kP_k^2(1+O(1/\ln T))$ determines $\kappa$,
and therefore
\begin{equation}
v_k^{*}=x_kP_k
=\frac{P_k^2}{\sum_jP_j^2}\,V\,\bigl(1+O(1/\ln T)\bigr),
\label{eq:budget-law-app}
\end{equation}
which is \eqref{eq:budget-law}. The hypothesis
$VP_{\max}/\sum_jP_j^2\le\tfrac34$ ensures the limiting allocation stays
in the window where the constants are uniform.
\end{proof}

\begin{table}[!htb]
\small\centering
\caption{\textbf{The budgeted optimum bends toward the $P^2$ law.} Two
modes, $P=(1,4)$, budget $V=5$, uniform grid; exact minimization of
$\sum_k\Phi_k$. The limit is
$P_1^2V/\sum_jP_j^2=\tfrac{5}{17}=0.294$, approached at the proved rate
$O(1/\ln T)$. Note that equal-budget protocols and free-noise
optimization answer different questions and have different exponents;
both should be reported.}
\label{tab:budget-bend}
\begin{tabular}{@{}rcccccc@{}}
\toprule
$T$ & $5$ & $10$ & $50$ & $200$ & $10^3$ & $2{\times}10^4$ \\
\midrule
$v_1^{*}$ & $0.773$ & $0.634$ & $0.460$ & $0.391$ & $0.347$ & $0.306$ \\
\bottomrule
\end{tabular}
\end{table}
\subsubsection{Proof of Theorem~\ref{thm:zidentity}}
\label{app:zidentity}

\begin{proof}
\textbf{(The identity.)} Let $z(\rho)=v\rho/\varphi(\rho)$. Substituting
into \eqref{eq:deficit-closed} through $\varphi=v\rho/z$ and
$1-z=(1-\rho)P/\varphi$ collapses the expression to
\begin{equation}
D_0=P\,F(z),
\qquad
F(z):=\sum_{i=1}^{T}\frac{(\Delta z_i)^2}{z_i},
\label{eq:z-deficit}
\end{equation}
with $z_0=0$ and $z_T=1$; the two-line simplification has symbolic
residual $0$. Equation~\eqref{eq:z-deficit} restates
\eqref{eq:z-identity}. Two identities used repeatedly below are
\begin{equation}
M(\rho)=P\,z(\rho),
\qquad
\frac{\Delta z_i}{z_i}=\frac{P\,\Delta\rho_i}{\rho_i\,\varphi(\rho_{i-1})},
\label{eq:Mz}
\end{equation}
both immediate from $K(\rho)=P/\varphi(\rho)$.

\textbf{(a) Color-blindness of the achievable set.} For every $v>0$ the
map $\rho\mapsto z$ is a continuous strictly increasing bijection of
$[0,1]$ (left panel of Figure~\ref{fig:app-z}). The achievable
$z$-grids --- and hence the schedule-optimized deficit --- therefore
coincide for every color at every finite $T$.

\textbf{(b) Uniqueness of the optimal grid.} $F$ is a sum of
quadratic-over-linear terms and so is jointly convex. If two minimizers
existed, $F$ would be affine on the segment joining them, which forces
every ratio $z_{i-1}/z_i$ to be constant along that segment. Together
with the pinned endpoints those ratios determine the grid, so the
minimizer is unique.

\textbf{(c) The optimal grid in closed form.} Put $a_i:=z_{i-1}/z_i$.
Stationarity of $F$ in the interior coordinates gives
\begin{equation}
2a_{i+1}=1+a_i^{2},
\qquad a_1=\frac{z_0}{z_1}=0,
\label{eq:arec}
\end{equation}
an explicit forward recurrence, from which
$z_i=\prod_{j>i}a_j$. Table~\ref{tab:zrec} lists the first values and
the resulting optimal costs. Brute-force minimization agrees to
$10^{-14}$ for $T=2,3,4$. At $T=2$ the exact optimum is
$z_1=\tfrac12$ with $F^{*}=\tfrac34$, whereas the continuum schedule
$z=s^2$ gives $\tfrac{13}{16}$: the continuum schedule is \emph{not}
exactly optimal at finite $T$.

\textbf{(d) Asymptotics of the optimum.} Set
$\varepsilon_i:=1-a_i$. Two exact rewrites start the argument. First
$\Delta z_i=z_i(1-a_i)$, so the optimal value is
\begin{equation}
F^{*}=\sum_{i=1}^{T}z_i\,\varepsilon_i^{2}.
\label{eq:Fstar}
\end{equation}
Second, \eqref{eq:arec} becomes \emph{exactly}
\begin{equation}
\varepsilon_{i+1}=\varepsilon_i\Bigl(1-\frac{\varepsilon_i}{2}\Bigr),
\qquad\varepsilon_1=1 ,
\label{eq:erec}
\end{equation}
a purely quadratic recurrence with no cubic term. It shows
$\varepsilon_i\in(0,1]$ and that the sequence decreases, so
\begin{align}
\frac{1}{\varepsilon_{i+1}}
&=\frac{1}{\varepsilon_i}
 +\frac{1}{2}\cdot\frac{1}{1-\varepsilon_i/2}\notag\\
&=\frac{1}{\varepsilon_i}+\frac12+\frac{\varepsilon_i}{4}
 +O(\varepsilon_i^2).
\label{eq:einv}
\end{align}

\emph{Sharp recurrence asymptotics.} Dropping the positive correction in
\eqref{eq:einv} gives $1/\varepsilon_i\ge(i+1)/2$, i.e.
$\varepsilon_i\le2/(i+1)$. Substituting that bound back yields
$\sum_{j<i}\varepsilon_j/4=\tfrac12\ln i+O(1)$ and
$\sum_{j<i}\varepsilon_j^2=O(1)$, so a two-sided induction with explicit
constants gives
\begin{align}
\frac{1}{\varepsilon_i}&=\frac{i}{2}+\frac12\ln i+O(1),\\
\varepsilon_i&=\frac{2}{i}
  \Bigl(1+O\bigl(\tfrac{\ln(i+1)}{i+1}\bigr)\Bigr).
\label{eq:easym}
\end{align}

\emph{Modal positions.} Since $z_i=\prod_{j>i}a_j$,
\begin{align}
\ln z_i
&=\sum_{j>i}\ln(1-\varepsilon_j)\notag\\
&=-\sum_{j>i}\varepsilon_j+O\Bigl(\sum_{j>i}\varepsilon_j^2\Bigr)\notag\\
&=-2\ln\frac{T}{i}+O\Bigl(\frac{\ln(i+1)}{i+1}\Bigr),
\end{align}
using $\sum_{j>i}2/j=2\ln(T/i)+O(1/i)$ and the fact that the two
correction sums $\sum_{j>i}(\ln j)/j^{2}$ and $\sum_{j>i}1/j^{2}$ are
both $O(\ln(i+1)/(i+1))$. Hence
\begin{equation}
z_i^{*}=(i/T)^2\,e^{\,O(\ln(i+1)/(i+1))} .
\label{eq:zasym}
\end{equation}
This is the rigorous convergence statement, with its rate: the exact
optimizer approaches the continuum schedule $z=s^2$. Equivalently the
limiting schedule is
\begin{equation}
\rho^{*}(s)=\frac{Ps^2}{Ps^2+v(1-s^2)},
\end{equation}
under which $M(\rho^{*}(s))=Ps^2$ by \eqref{eq:Mz}: the conditional
MMSE is quadratic in solver time.

\emph{The floor.} Combining \eqref{eq:easym} and \eqref{eq:zasym} term
by term,
\begin{align}
z_i\varepsilon_i^{2}
&=\frac{4}{T^{2}}
  \Bigl(1+O\bigl(\tfrac{\ln(i+2)}{i+1}\bigr)\Bigr),\\
F^{*}
&=\frac{4}{T}
  +O\Bigl(\frac{1}{T^{2}}\sum_{i=1}^{T}\frac{\ln(i+2)}{i+1}\Bigr)
 \notag\\
&=\frac{4}{T}\Bigl(1+O\bigl(\tfrac{\log^{2}T}{T}\bigr)\Bigr),
\end{align}
so $D_0^{*}=\tfrac{4P}{T}\bigl(1+O(\log^{2}T/T)\bigr)$. The $\log^{2}$
factor comes from $\sum_i(\ln i)/i$; improving the rate would require
cancellation among the per-term corrections, and no such cancellation is
evident. The last column of Table~\ref{tab:zrec} supports the proved
law.

Finally, for the uniform $s$-grid under $z=s^2$ the summand of
\eqref{eq:z-deficit} simplifies symbolically to
$P(s_i^2-s_{i-1}^2)^2/s_i^2$, in which the color $v$ disappears
identically --- confirming (a) term by term.
\end{proof}

\begin{table}[!htb]
\small\centering
\caption{\textbf{The optimal $z$-schedule of
Theorem~\ref{thm:zidentity}(c).} \emph{Left:} the first ratios
$a_i=z_{i-1}/z_i$ from the forward recurrence \eqref{eq:arec}, as exact
rationals. \emph{Right:} the optimal cost, exactly for small $T$ and
against the floor $F^{*}=4/T$ of part~(d). The approach to $1$ is at the
proved rate $1+O(\log^2T/T)$.}
\label{tab:zrec}
\begin{tabular}{@{}cl@{\hskip 1.2em}rll@{}}
\toprule
$i$ & $a_i$ & $T$ & $F^{*}$ & $F^{*}T/4$ \\
\midrule
$1$ & $0$                 & $2$    & $3/4$                & $0.375$ \\
$2$ & $1/2$               & $3$    & $39/64$              & $0.457$ \\
$3$ & $5/8$               & $4$    & $16929/32768$        & $0.517$ \\
$4$ & $89/128$            & $10$   & $0.2778$             & $0.694$ \\
$5$ & $24305/32768$       & $200$  & $0.01931$            & $0.966$ \\
$6$ & $0.7751\ldots$      & $10^4$ & $3.996{\times}10^{-4}$ & $0.999$ \\
\bottomrule
\end{tabular}
\end{table}

\begin{figure*}[!htb]
\centering
\includegraphics[width=\textwidth]{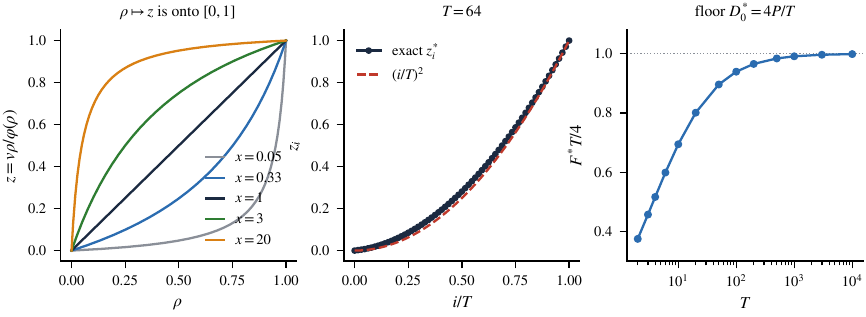}
\caption{\textbf{The solver coordinate $z$ absorbs the color.}
\emph{Left:} $\rho\mapsto z=v\rho/\varphi(\rho)$ for five values of
$x=v/P$. Each is a bijection of $[0,1]$, so every $z$-grid is reachable
at every color --- this is Theorem~\ref{thm:zidentity}a.
\emph{Center:} the exact optimizer $z_i^{*}$ from the recurrence
\eqref{eq:arec} against the continuum law $(i/T)^2$ at $T=64$,
illustrating \eqref{eq:zasym}. \emph{Right:} $F^{*}T/4\to1$, the floor
$D_0^{*}=4P/T$ of part~(d).}
\label{fig:app-z}
\end{figure*}

\subsubsection{The multi-objective schedule frontier}
\label{app:frontier}

The discretization-optimal schedule has gain-error susceptibility
\begin{equation}
\Xi:=\sum_i\dzf{i}\,(1-z_i)=2\ln T+O(1),
\label{eq:susc}
\end{equation}
which diverges logarithmically; numerically
$\Xi=2\ln T-3.05$ at $T=10^4$. The two objectives therefore conflict:
discretization favors many small steps near $z=0$, whereas robustness
favors a short log-path. Under model error, reference design becomes a
multi-objective $z$-schedule problem --- minimize $\sum_i(\Delta z_i)^2/z_i$
subject to a susceptibility budget. The constraint involves the ratios
$z_{i-1}/z_i$, which are \emph{not} jointly convex, so convexity of the
exact finite-$T$ frontier does not follow from convexity of $F$. In the
continuum, however, the frontier is available in closed form.

\begin{theorem}[Continuum susceptibility--discretization frontier]
\label{thm:frontier}
Consider increasing $C^1$ schedules $z:[0,1]\to(0,1]$ with
$z(0^{+})=z_1$ and $z(1)=1$.

\textbf{(i) Susceptibility is path-independent.}
\begin{equation}
\Xi^{\mathrm{cont}}
=\int_0^1\frac{\dot z}{z}(1-z)\dd s
=\int_{z_1}^{1}\frac{1-z}{z}\dd z
=\ln\frac{1}{z_1}-(1-z_1).
\end{equation}
This depends only on the launch point $z_1$. The budget
$\Xi^{\mathrm{cont}}\le\xi$ therefore imposes only $z_1\ge a(\xi)$ and
no further path constraint, where $a(\xi)\in(0,1)$ is the unique root of
$\ln\tfrac1a-(1-a)=\xi$.

\textbf{(ii) The square law, derived.} Substituting $w=\sqrt z$ turns
the discretization functional into a Dirichlet energy,
\begin{equation}
\int_0^1\frac{\dot z^2}{z}\dd s=4\int_0^1\dot w^2\dd s ,
\end{equation}
with endpoints fixed at $w(0)=\sqrt a$ and $w(1)=1$. Cauchy--Schwarz
gives the unique minimizer $w(s)=\sqrt a+(1-\sqrt a)s$, that is the
free square-law path launched at $a$:
\begin{equation}
z(s)=\bigl(\sqrt a+(1-\sqrt a)s\bigr)^2,
\qquad
E(\xi)=4\bigl(1-\sqrt{a(\xi)}\bigr)^2 .
\end{equation}

\textbf{(iii) Closed-form, $C^1$ frontier.} Differentiating the
constraint,
\begin{align}
\frac{\dd a}{\dd\xi}&=-\frac{a}{1-a},\\
\frac{\dd^2a}{\dd\xi^2}&=\frac{a}{(1-a)^{3}}>0,\\
\frac{\dd E}{\dd\xi}&=\frac{4\sqrt a}{1+\sqrt a}>0 .
\end{align}
Both derivatives are strictly monotone in $\xi$: $a$ decreases as $\xi$
increases, while $\sqrt a/(1+\sqrt a)$ increases with $a$. Each branch
therefore has curvature of one sign. The launch point $a(\xi)$ is
$C^1$, strictly decreasing and strictly convex; the smooth-path energy
$E(\xi)$ is $C^1$, strictly increasing and strictly concave, and
approaches the free value $4$ as $\xi\to\infty$.
\end{theorem}

At finite $T$ the launch point $z_1$ is exactly the first step's $O(1)$
contribution to the deficit, because its term in \eqref{eq:z-deficit}
is $(z_1-0)^2/z_1=z_1$. Along the envelope family the total deficit is
therefore
\begin{equation}
a(\xi)+\tfrac1TE(\xi)\bigl(1+o(1)\bigr),
\end{equation}
whose leading-order frontier is convex, decreasing and available in
closed form; the concave $E$-branch contributes only the $O(1/T)$
correction.

The susceptibility constraint becomes inactive near $\xi\approx2\ln T$,
where the frontier flattens at the free optimum of
Theorem~\ref{thm:zidentity}d. The continuum identity in (i) is the main
conceptual statement and replaces the discrete bound
$\Xi\le\ln(1/z_1)$; the discrete bound remains useful as a finite-$T$
comparison, and the difference between them, $1-z_1$, is exactly the
first step's susceptibility contribution.


\subsection{Proofs for Section~\ref{sec:modelerror} (Model Error and Learning)}
\label{app:thmC}
\subsubsection{Completeness of the error model}
\label{app:error-model}

In the solvable model the exact drift is affine,
$\hat x_0=Wx_1+K_i(x-\bar c_ix_1)$. Any learned affine drift can differ
from it in exactly two ways: a relative gain error $\eta_i$, which
mis-scales the sensitivity to the state deviation, and an additive bias
$\beta_i$. Because the analysis conditions on $x_1$, the bias may depend
on $x_1$; consequently an error in the drift's $x_1$-coefficient --- for
instance using $\hat W\ne W$ --- is absorbed into $\beta_i$. The
two-channel decomposition is therefore exhaustive within the model.
Both channels preserve the affine--Gaussian structure, so the terminal
laws remain exact.

\subsubsection{Proof of Proposition~\ref{prop:gain-unbiased}}
\label{app:gain-proof}

\begin{proof}
The perturbed step gives
\begin{equation}
c'=\bigl[r+(1-r)K(1+\eta)\bigr]c
   +(1-r)\bigl[W-K(1+\eta)\bar c\bigr].
\end{equation}
Substituting the inductive value $c=\bar c_i$, the two terms
$K(1+\eta)\bar c$ cancel \emph{identically in $\eta$}, leaving
$c'=r\bar c_i+(1-r)W=\bar c_{i-1}$. Mis-scaling the deviation around
the correct center therefore cannot bias the mean; only the additive
bias can move it.
\end{proof}

\subsubsection{Proof of Theorem~\ref{thm:colorblind}}
\label{app:colorblind-proof}

\begin{proof}
Work in $z$-coordinates. Three direct computations give exact step
quantities that do not depend on color at matched $z$-grids. The
unperturbed contraction is a ratio of values of
$w(z):=z+x(1-z)$, consistent with Lemma~\ref{lem:telescope}, and the
key identity is the second half of \eqref{eq:Mz},
\begin{equation}
\frac{(1-r_i)K_i}{A_i}=\dzf{i},
\label{eq:keyid}
\end{equation}
which is exact and independent of $x=v/P$. The gain-perturbed
contraction is therefore
\begin{equation}
\hat A_i=A_i\Bigl(1+\eta_i\dzf{i}\Bigr).
\end{equation}

The bias enters at step $i$ as $(1-r_i)\beta_i$ and is multiplied by the
downstream contractions $\hat A_{i-1}\cdots\hat A_1$ on its way to the
output, so its weight is
\begin{equation}
(1-r_i)\prod_{m<i}\hat A_m
=\dzf{i}\prod_{m<i}\Bigl(1+\eta_m\dzf{m}\Bigr),
\end{equation}
where we used $\prod_{m<i}A_m=P/\varphi(\rho_{i-1})$ from
Lemma~\ref{lem:telescope} together with
$(1-r_i)P/\varphi(\rho_{i-1})=(1-r_i)K_i/A_i=\dzf{i}$ from
\eqref{eq:keyid}. When $\eta\equiv0$ the weight reduces to $\dzf{i}$.
In general, therefore,
\begin{equation}
m_0=\sum_i\beta_i\,\dzf{i}
    \prod_{m<i}\Bigl(1+\eta_m\dzf{m}\Bigr),
\label{eq:pert-mean}
\end{equation}
which is \eqref{eq:perturbed-mean}. Unrolling the variance recursion
\eqref{eq:Vrec} with the same factors gives
\begin{equation}
\frac{V_0}{P}=\sum_{j}z_{j-1}\dzf{j}
  \prod_{m<j}\Bigl(1+\eta_m\dzf{m}\Bigr)^{2},
\label{eq:pert-var}
\end{equation}
which is \eqref{eq:perturbed-variance}. In the special case
$\eta\equiv0$ this becomes
\begin{equation}
\frac{V_0}{P}=\sum_jz_{j-1}\dzf{j}
=1-\sum_j\frac{(\Delta z_j)^2}{z_j},
\end{equation}
using $z_{j-1}=z_j-\Delta z_j$ and $\sum_j\Delta z_j=1$, which recovers
\eqref{eq:z-deficit}.

Every factor in \eqref{eq:pert-mean}--\eqref{eq:pert-var} depends only
on the $z$-grid and the error profile. The terminal law is therefore
identical for every $v>0$, and with mean error the terminal KL is
$\tfrac12(u+m_0^2/P-1-\ln u)$ with $u=V_0/P$ by (S6).
\end{proof}

\begin{remark}[Numerical verification]
\label{rem:colorblind-check}
For simultaneous perturbations $\eta(z)=0.2z+0.1\sin3z$ and a smooth
bias profile $\beta(z)$, both closed forms agree with the exact
recursion to $\le3\times10^{-16}$ in $V_0$ and in $m_0$. At a matched
$z$-grid they are also identical to $12$ digits across
$x\in[0.01,100]$.
\end{remark}

\subsubsection{Proof of Proposition~\ref{prop:learning-null}}
\label{app:learning-null}

\begin{proof}
At level $z$ the conditional correlation satisfies
\begin{equation}
\corr^2\bigl(x_0,\;\text{deviation}\mid x_1\bigr)=1-z
\end{equation}
exactly, with symbolic residual $0$. Up to scale, the per-mode joint law
of $(x_0,x_t)$ given $x_1$ therefore depends only on $z$, so any
scale-equivariant learner has a color-free relative-error law at each
$z$-level. For zero-intercept least squares from $n$ pairs on a normalized design
$\sum_iX_i^2=n\tau^2$,
\begin{equation}
\Var\bigl(\hat K/K-1\bigr)=\frac{M}{n\Sigma K^2}=\frac{z}{n(1-z)},
\end{equation}
which is again color-free; under a random design $n$ is replaced by
$n-2$. Combined with
Theorem~\ref{thm:colorblind}, per-mode color is exactly invisible end
to end once the $z$-schedule is fixed.
\end{proof}

\begin{remark}[Correlated miscalibration]
\label{rem:miscal}
A constant relative gain error compounds along the $z$-path as
\begin{equation}
\exp\Bigl(\pm2\eta\sum_i\dzf{i}\Bigr).
\end{equation}
On a \emph{fixed} $\rho$-grid the color
changes the shape of that path, and the measured optima move a great
deal (Table~\ref{tab:miscal}). By Theorem~\ref{thm:colorblind} this
dependence must be read as a $z$-schedule effect, not as a color effect
in its own right.
\end{remark}

\begin{center}
\small
\captionof{table}{\textbf{Constant gain error $\pm\eta$ on a fixed
$\rho$-grid} moves the apparent optimum by orders of magnitude ---
an artifact of the induced $z$-schedule, per
Remark~\ref{rem:miscal}.}
\label{tab:miscal}
\begin{tabular}{@{}lcccc@{}}
\toprule
$\eta$ & $0$ & $0.02$ & $0.05$ & $0.2$ \\
\midrule
$x^{*}$ & $0.34$ & $5.4$ & $16.4$ & $78.8$ \\
\bottomrule
\end{tabular}
\end{center}

\subsubsection{Ridge whitening (Proposition~\ref{prop:ridge})}
\label{app:ridge}

A shared ridge penalty $\lambda$ shrinks the estimated gain through the
relative bias
\begin{equation}
\eta_k(z)=-\frac{\lambda}{\lambda+n\Sigma_k(z)} .
\label{eq:ridge-eta}
\end{equation}
Because $\Sigma_k$ is dimension\emph{ful}, this penalty breaks the scale
symmetry of Theorem~\ref{thm:proportional}. Exact-chain
optimization with $T{=}100$ and $\lambda{=}1$ gives the sweep in
Table~\ref{tab:ridge-sweep}, which decreases strongly with the
effective sample size $nP$.

\begin{table*}[t]
    \small\centering
    \caption{\textbf{Ridge-whitening sweep.} Optimal scale $x^{*}(nP)$
    for $T=100$ and $\lambda=1$. Smaller effective sample size $nP$
    produces a flatter optimal spectrum.}
    \label{tab:ridge-sweep}
    \begin{tabular}{rccccccc}
        \toprule
        $nP$ & $10$ & $30$ & $100$ & $300$ & $10^3$ & $10^4$ & $\infty$ \\
        \midrule
        $x^{*}(nP)$ & $11.92$ & $3.95$ & $1.58$ & $0.89$ & $0.58$ & $0.39$ & $0.363$ \\
        \bottomrule
    \end{tabular}
\end{table*}

For $P=(1,4)$ and $n=100$ the optimum satisfies
\begin{equation}
\frac{v_2^{*}}{v_1^{*}}=4\cdot\frac{x^{*}(400)}{x^{*}(100)}=1.99,
\end{equation}
against the proportional value $4$: weakly observed frequencies need
disproportionately \emph{more} reference noise to counteract shrinkage,
and the regularized optimum is substantially flatter. With a shared
\emph{constant} $\lambda$ the modes remain separable, because the
mechanism is per-mode scale sensitivity. A genuinely shared network
introduces true cross-mode coupling; the simplest such model ties one gain
across modes, and its exact cost equals the curvature-weighted dispersion
of the per-mode optimal gains. The fully nonlinear shared-network case
remains open.

\paragraph{Data-dependent optima.}
Using exact moment recursions and no Monte Carlo, the leading-order
statistical contribution equals $1/n$ times the discretization
functional; measured/predicted is $0.971$ and $0.979$ at $x=0.3$ and
$x=2.0$. The mean effect therefore does not move the argmin. The shift
is caused instead by the variance-of-variance penalty, which produces
the mild downward trend
\begin{equation}
x^{*}=0.323\;(n{=}\infty)\;\to\;0.305\;(n{=}300)\;\to\;0.272\;(n{=}100).
\end{equation}

\subsubsection{Distortion--perception and underdispersion correction}
\label{app:dp}

\begin{corollary}[Distortion--perception trade-off, derived]
\label{cor:dp}
The posterior-mean estimator $Wx_1$ has conditional MSE exactly $P$ for
every $T$, independent of the reference; this is the only quantity made
invariant by mean exactness (Lemma~\ref{lem:mean-exact}). The sampled
output has the correct mean and is conditionally independent of the true
$x_0$ given $x_1$, so its conditional MSE is
\begin{equation}
\mathrm{MSE}=P+V_0\;\in\;[P,\,2P),
\end{equation}
which depends on the reference through $V_0$ alone. The collapsed
sampler minimizes it at $V_0=0$, where it is the posterior-mean
estimator in disguise; at distributional perfection $V_0=P$ and the MSE
equals $2P$. The classical factor-2 ($3$\,dB) cost of posterior
sampling therefore follows as a theorem, and the reference traces the
entire trade-off curve.
\end{corollary}

\begin{corollary}[Principled underdispersion correction]
\label{cor:underdisp-fix}
By Corollary~\ref{cor:underdispersion} the plug-in sampler's only
failure mode is underdispersion, so deliberate gain inflation can serve
as a calibrated correction: by \eqref{eq:pert-var}, choose $\eta(z)$ so
that
\begin{equation}
\sum_{j}z_{j-1}\dzf{j}
  \prod_{m<j}\Bigl(1+\eta_m\dzf{m}\Bigr)^{2}=1 .
\label{eq:calibrate}
\end{equation}
Each term carries its own partial product; there is no single global
factor multiplying the $\eta\equiv0$ variance. A convenient
one-parameter solution holds $\eta$ constant and solves the resulting
scalar equation, giving an exact-formula counterpart to temperature and
churn heuristics.
\end{corollary}

Note that an
objective that checks only the expected variance can be satisfied by
\emph{any} error that adds variance. The appropriate objective is
$\E[\KL]$ over error realizations.
\section{Experimental Protocols \& Additional Results}
\label{app:experiments}

\subsection{Reference color in $I^2$SB super-resolution}
\label{sec:exp-i2sb}

We begin with the question that motivated the theory: does reference
color matter for a full pretrained restoration model? Starting from the
stock $I^2$SB $4\times$ super-resolution checkpoint, we fine-tune one
model per reference for $50$k steps under identical optimization
settings, changing only the color of the reference noise --- white,
low-pass ($1/f^{2}$), and high-pass ($f^{2}$). All three references
share the same scale $\varepsilon$ and the same total pixel variance,
which removes the trivial confound; only the allocation of variance
across frequencies differs. Since $4\times$ downsampling destroys high
spatial frequencies, the high-pass reference is the one that
concentrates noise on the destroyed information --- the qualitative
direction of $v_k\propto P_k$.

\begin{table}[!htb]
\small\centering
\caption{\textbf{$4\times$ super-resolution with recolored references.}
PSNR/SSIM on held-out aligned pairs at NFE $20$; equal data, compute and
total reference variance. ``Best'' selects the best validation
checkpoint; ``final'' is at $50$k steps.}
\label{tab:i2sb-color}
\begin{tabular}{@{}lcccc@{}}
\toprule
& \multicolumn{2}{c}{best} & \multicolumn{2}{c}{final} \\
\cmidrule(lr){2-3}\cmidrule(lr){4-5}
reference & PSNR$\uparrow$ & SSIM$\uparrow$ & PSNR$\uparrow$ & SSIM$\uparrow$ \\
\midrule
high-pass ($f^{2}$)   & $\mathbf{29.13}$ & $\mathbf{0.714}$ & $\mathbf{27.66}$ & $\mathbf{0.647}$ \\
white                 & $27.64$ & $0.579$ & $26.82$ & $0.554$ \\
low-pass ($1/f^{2}$)  & $26.36$ & $0.476$ & $21.35$ & $0.207$ \\
\bottomrule
\end{tabular}
\end{table}

Table~\ref{tab:i2sb-color} shows the ordering predicted by the
destroyed-information picture: high-pass $>$ white $>$ low-pass on both
metrics, at both the best checkpoint and the end of training. The
high-pass reference beats white by $1.5$~dB PSNR and $0.14$ SSIM at the
best checkpoint, despite the base model having been pretrained with
white noise. The low-pass reference is more striking: its best
checkpoint is its \emph{first} one, and continued fine-tuning is
monotonically destructive, losing $5.0$~dB by $50$k steps. A reference
that concentrates noise on frequencies the observation already preserves
does not merely learn slowly; it steadily erases the pretrained model's
ability to reconstruct.

This experiment changes only the color of the noise, at a single scale
$\varepsilon$ and a single step budget. It confirms the coarse direction
the theory predicts --- noise belongs where the observation destroyed
information --- but it does not test the predicted amount $x^{*}(T)$,
which requires sweeping the scale (Fig.~\ref{fig:teasor} does this in
the exact setting). Even within these limits the effect is large: on a
real pretrained restoration model the reference spectrum is a
first-order design choice, not an implementation detail.

\subsection{Protocols}
\label{app:protocols}

\subsubsection{Exact numerics (Phase 1)}
\label{app:phase1}

Every quantity is computed three ways --- the closed-form deficit
\eqref{eq:deficit-closed}, the telescoped $V_0$, and the exact affine
recursion \eqref{eq:crec}--\eqref{eq:Vrec} --- and the three are
cross-checked to machine precision before each run. The fixed
regression anchor is $D_0=1.171587$ at $P=v=4$ and $T=10$.

Figure~\ref{fig:exact}(a) uses $P=v=4$ on uniform and power grids with
$T$ up to $10^5$. Figure~\ref{fig:exact}(c) locates $x^{*}(T)$ by
golden-section minimization of $\Phi$ at $22$ logarithmically spaced
$T\in[5,10^5]$. Figure~\ref{fig:exact}e uses the two-cluster grid at
$T{=}48$; each candidate minimum is re-verified with exact derivatives
in 50-digit \texttt{mpmath} arithmetic. Table~\ref{tab:twomode} uses the
exact recursion with budget $\sum_kv_k=5$, a linear schedule and vertex
$\epsilon=10^{-2}$. The $z$-schedule computations use the forward
recurrence \eqref{eq:arec}. All Phase-1 experiments are deterministic,
use no random seeds, and run in minutes on a laptop.

\subsubsection{Gaussian Setting (Phase 2)}
\label{app:phase2}

The synthetic setting has $64$ modes (up to $256$ in secondary sweeps),
prior $S_k=k^{-2}$, Gaussian MTF $h_k=\exp(-(k/k_c)^2)$ with $k_c=16$,
and flat noise $N_k=10^{-3}$; $P_k$ is computed exactly and saved with
every run. The matched reference is $v_k=x^{*}(T)P_k$ with $x^{*}$ from
Table~\ref{tab:xstar}; white, anti-matched ($\propto1/P_k$) and
prior-colored ($\propto S_k$) references are rescaled to the matched
total budget.

The predictor is $64$ \emph{independent} per-mode MLPs (batched with
\texttt{bmm}), each receiving $(x_t,x_1,\text{32-dim sinusoidal }t)$,
with $3$ SiLU hidden layers and scalar output $\hat x_0$. Training:
Adam $10^{-3}$, batch $2048$--$4096$, $12$k--$20$k steps, EMA $0.999$,
continuous $t\sim U(0,1)$, per-mode I/O normalization by $\sqrt{S_k}$ (a
benign per-mode rescaling), seeds $\{0,1,2\}$, common random numbers
across references.

Terminal KL is estimated as follows: conditional on $x_1$, the terminal
law of each mode is fit by regression over a common-random-number batch
of $10^5$ samples. The KL divergence between the fitted Gaussian and
$\N(Wx_1,P)$ is then evaluated in closed form; the empirical KL floor
is about $10^{-5}$.
The ridge-whitening and constant-sweep
experiments use closed-form ridge/OLS learners per
$(k,t\text{-bin})$ with $32$ bins and exact deterministic chain
propagation, hence zero sampling noise. The Jacobian probe estimates
$\hat\eta(z)$ by finite differences of $\partial\hat x_0/\partial x_t$
on the trained fingerprint network and applies \eqref{eq:pert-var}.

\subsubsection{FFHQ $64{\times}64$ (Phase 3)}
\label{app:phase3}

FFHQ is center-cropped and resized to $64{\times}64$ with Lanczos
interpolation; the first $60$k images are used for training, for fitting
$S(k)$ and for FID statistics, and the next $10$k for evaluation; pixels
lie in $[-1,1]$.

The degradation is defined in the DFT domain so that $P(k)$ is analytic:
$x_1=k_{\mathrm{blur}}*x_0+n$ with Gaussian blur
$\sigma_{\mathrm{blur}}=2.0$\,px, transfer function
$h(k)=\exp(-2\pi^2\sigma_{\mathrm{blur}}^2|f|^2)$,
and noise $\sigma_n=0.05$, i.e.\ $N=2.5\times10^{-3}$. There is no
decimation and hence no aliasing block, so
Corollary~\ref{cor:blocks} is not needed here. $S(k)$ is fit once by
radially averaging the training-split power spectrum and saved in a
bundle together with $P(k)$, $W(k)$, $h(k)$ and $x^{*}(T)$. The matched
reference uses $x^{*}(T_{\mathrm{ref}}{=}50)=0.404$; all budget-matched
colors have the same per-pixel noise variance, so ``same noise level,
different spectrum'' is literal.

Each run trains one $I^2$SB-style ADM U-Net (base width $128$,
multipliers $(1,2,2,2)$, $2$ residual blocks per resolution, attention
at $16^2$ and $8^2$; $39.6$M parameters), conditioned by concatenating
$x_1$ with $x_t$ ($6$ input channels) plus a standard $t$-embedding.

\emph{Original recipe:} $\hat x_0$ MSE with $\rho\sim U(0,1)$ on the
uniform schedule; Adam $10^{-4}$, batch $128$, $150$k steps, EMA
$0.9999$, dropout $0.1$, horizontal flips, bf16 autocast. Matched and
white use seeds $\{0,1,2\}$; anti, $\alpha{=}1$, $\alpha{=}2$ and
equal-budget matched use seed $0$ --- $10$ runs, each $4$--$7$\,h on one
modern GPU.

Evaluation uses the ancestral plug-in-mean sampler at
NFE $5$, $10$, $20$, $50$ and $100$ on uniform sub-grids with fixed generator
seeds and chunking, so comparisons across references use byte-identical
common random numbers. We report PSNR/SSIM/LPIPS on $1$k held-out
images and FID against FFHQ-64 statistics; radially averaged error
spectra include the exact-drift overlay $2P(k)-D_0(k)$. Every $5$k
steps the EMA network is probed on $256$ fixed images over a $16$-point
$t$-grid, recording the per-(radial band, $t$) fingerprint relative to
the Bayes floor.

\subsubsection{Converged recipe and the $\theta$-family}
\label{app:phase3-conv}

The converged recipe is identical to the original except dropout $0$,
batch size $512$ and $300$k steps --- eight times as many gradient
samples. The $\theta$-family is $v(k)\propto P(k)^{\theta}$ rescaled to
the fixed total budget $A$, with
$\theta\in\{0,\tfrac14,\tfrac12,\tfrac34,1\}$; the endpoints coincide
with the white and matched references at that budget (verified against
the saved bundle to $10^{-6}$, so separate endpoint runs are
unnecessary). Runs: matched and white with seeds $\{0,\dots,4\}$;
$\theta\in\{0.25,0.5,0.75\}$ with seed $0$; matched and white at
$\sigma_{\mathrm{blur}}=4.0$ with seed $0$ (separate bundle).

\emph{FID protocol.} Every FID in Section~\ref{sec:exp-regime} uses
$50$k samples, formed by pooling $5$ replicates of the $10$k evaluation
set with fresh sampler driving noise; per-replicate $10$k FIDs are
stored, and the original-recipe runs were re-evaluated under the same
protocol before any comparison.

\emph{Schedule ablation.} Every trained network is also evaluated on the
$z$-optimal grid of Theorem~\ref{thm:zidentity}(c) without retraining.

Predictions P1--P4 of Section~\ref{app:regime-details} were written down
before the first converged-recipe run began.

\subsection{Exact numerical evaluation}
\label{app:exact-details}

\paragraph{Numerical verification.}
All deterministic-recursion self-tests pass to machine precision. For
the invisibility experiment (Fig.~\ref{fig:exact}a) the terminal
variance reaches $V_0=3.9995$ at $T=10^5$ for $P=v=4$, while the
singular case $v=0$ collapses, as Theorem~\ref{thm:invisibility}(b)
requires. For the scaling experiment (Fig.~\ref{fig:exact}c) the
measured optima include $x^{*}(217)=0.329$ and $x^{*}(10^5)=0.212$, and
the ratio $x^{*}(T)/(2\ln T)^{-1/2}$ decreases from $1.29$ at $T=5$ to
$1.017$ at $T=10^5$; Table~\ref{tab:xstar} gives the full sweep. For the
clustered $T=48$ schedule (Fig.~\ref{fig:exact}e) seven local minima are
certified in 50-digit arithmetic.

\begin{table*}[!htb]
\small\centering
\caption{\textbf{Finite-step KL in a two-mode setting.} Total KL
$\sum_k\KL_k$ for $P=(1,4)$, total budget $5$ and a linear schedule.
The matched allocation is best at every tested step count. The vertex
allocation uses $v_2=\epsilon$ with $\epsilon=10^{-2}$ and is the
allocation selected by the degenerate Bayes-risk objective of
App~\ref{app:deadroute}.}
\label{tab:twomode}
\begin{tabular}{@{}lccccc@{}}
\toprule
allocation & $T{=}5$ & $T{=}10$ & $T{=}20$ & $T{=}50$ & $T{=}200$ \\
\midrule
matched $(1,4)$              & $\mathbf{0.1534}$ & $\mathbf{0.0537}$ & $\mathbf{0.0184}$ & $\mathbf{0.0043}$ & $\mathbf{0.0004}$ \\
uniform $(2.5,2.5)$          & $0.1901$ & $0.0710$ & $0.0263$ & $0.0068$ & $0.0008$ \\
anti $(4,1)$                 & $0.2682$ & $0.1012$ & $0.0382$ & $0.0104$ & $0.0013$ \\
vertex $(5{-}\epsilon,\epsilon)$ & $1.8491$ & $1.3075$ & $0.8995$ & $0.4869$ & $0.1174$ \\
\bottomrule
\end{tabular}
\end{table*}

Table~\ref{tab:twomode} shows that the matched allocation minimizes KL
at every tested $T$; at $T=200$ the vertex allocation has $294\times$
larger KL.

\paragraph{Schedule and budget effects.}
For the $z$-schedule recurrence, $F^{*}T/4$ equals $0.375$, $0.694$,
$0.966$ and $0.999$ at $T=2$, $10$, $200$ and $10^4$
(Table~\ref{tab:zrec}), approaching the $4P/T$ floor of
Theorem~\ref{thm:zidentity}d. After schedule optimization the deficit is
invariant to reference color up to $10^{-12}$, as
Theorem~\ref{thm:zidentity}a requires. The budgeted optimum bends toward
the $P^2$ allocation predicted by Proposition~\ref{prop:budget}:
$v_1^{*}$ moves from $0.773$ at $T=5$ to $0.347$ at $T=10^3$ and
$0.306$ at $T=2\times10^4$, approaching the vertex value
$5/17\approx0.294$ (Table~\ref{tab:budget-bend}).

\subsection{Controlled Gaussian experiments}
\label{app:gaussian-details}

\paragraph{Tested predictions.}
We test whether: (i) the converged loss matches the Bayes floor
$P\phi/(\phi+P)$ with $\phi=v\rho/(1-\rho)$; (ii) the predicted
finite-NFE reference ordering appears; (iii) the optimal scale follows
the theoretical constant and bends toward $P^2$ under an equal-budget
constraint; (iv) finite data and weight decay flatten the optimum toward
white; and (v) a Jacobian-based estimate from one trained network
predicts terminal variance without retraining.

\paragraph{Loss fingerprint and NFE dependence.}
After $12$k training steps the median relative excess over the Bayes
floor is $0.18\%$ across the $64\times32$ grid of mode--time pairs,
below the pre-specified $5\%$ gate. For every NFE up to $50$ the KL
ordering is matched $<$ white $<$ anti-matched $<$ prior-colored; at
NFE $50$ the total KL is $0.19{\pm}0.03$, $0.21{\pm}0.03$,
$0.44{\pm}0.04$ and $4.3{\pm}0.3$ over three seeds.

At NFE $200$ gain error dominates discretization error: the learned
curves lie one to two orders of magnitude above their exact-drift
floors, the prior-colored reference becomes unstable
(KL $10^4$--$10^5$), and matched and white are statistically
indistinguishable at $0.60{\pm}0.16$ and $0.54{\pm}0.15$. This regime
therefore does not preserve the finite-step ordering predicted under
exact drift --- exactly as Theorem~\ref{thm:colorblind} anticipates.

\paragraph{Optimal scale and ridge whitening.}
The measured optimal constants are $0.575$ at $T=10$ and $0.402$ at
$T=50$, against the theoretical values $0.577$ and $0.404$ of
Table~\ref{tab:xstar}. At $T=200$ the measured optimum is $0.360$, which
is $8\%$ above the exact-drift value $0.332$ --- the direction predicted
for finite-data regularization by Proposition~\ref{prop:ridge}.
In the exact-chain ridge sweep at $T=100$ and $\lambda=1$ (relative
bias \eqref{eq:ridge-eta}), the optimal scale $x^{*}(nP)$ grows
monotonically as the effective sample size shrinks --- from $0.363$ at
$nP=\infty$ through $0.58$ at $nP=10^3$ to $11.92$ at $nP=10$ --- i.e.\
a smaller effective sample size produces a flatter optimal spectrum.
For $P=(1,4)$ and $n=100$ the ridge-adjusted optimum gives
$v_2^{*}/v_1^{*}=1.99$ against $4$ under exact proportionality: the
regularized optimum is substantially flatter.

\paragraph{Jacobian and schedule probes.}
The Jacobian analysis of the trained fingerprint network predicts
terminal variance with $0.6\%$ median relative error and needs no
retraining, validating the measurement protocol of
Section~\ref{sec:modelerror}. Per-mode schedule adaptation reduces
reference-color gaps toward the theoretical floor, supporting the
color-blindness prediction of Theorem~\ref{thm:zidentity}a once
schedule error is removed.

\subsection{FFHQ: evaluation details and secondary results}
\label{app:ffhq-details}

\paragraph{Data and metrics.}
All FFHQ runs use common random numbers across reference choices. PSNR,
SSIM and LPIPS are computed on the same $1{,}000$ held-out images. In
the original-recipe experiments FID is computed from $50$k restorations
against Inception statistics of the $60$k clean training images; the
converged-recipe experiments use $50$k restorations
(Section~\ref{app:regime-details}). The matched reference is computed
from the known Gaussian MTF and flat observation-noise spectrum; it is
not selected using validation performance.

\paragraph{Budget conventions.}
Most rows of Table~\ref{tab:ffhq} use total budget
$A=\sum_kx^{*}P(k)$. The ``matched, equal budget'' row instead uses the
optimal-white budget $B$. At NFE $50$ the equal-budget matched model has
FID $9.12$ against $8.45$ for matched under its free budget --- the
budget effect predicted by Proposition~\ref{prop:budget}. We report both
conventions because they answer different questions. The flatter
empirical optimum is also not specific to NFE $50$: at NFE $100$, white
attains FID $6.07$ against $6.64$ for matched, beyond the observed seed
variation.

\paragraph{Schedule re-evaluation.}
Each trained model is re-evaluated with the $z$-optimal schedule of
Theorem~\ref{thm:zidentity}, without retraining. The optimized schedule
improves every reference by a similar amount and leaves the ordering
unchanged (at NFE $50$: white $7.87\to7.10$, matched $8.45\to7.62$,
anti-matched $14.1\to13.0$); the white--matched FID gap changes only
from $0.58$ to $0.53$. It therefore removes a shared discretization
component but does not explain the color-dependent gap.

\subsection{Training regime and reference exponent}
\label{app:regime-details}

\paragraph{FID measurement.}
For each converged-recipe run, five independent groups of $10$k
restorations are generated with fresh sampler noise and their Inception
features pooled into a $50$k estimate; the median per-group spread is
$\pm0.04$ and the maximum $\pm0.08$. For direct comparison with the
original $10$k protocol, per-group $10$k FIDs are also reported. Pooling
does not change the conclusions: at NFE $50$ the matched--white gap is
$1.169$ on the $10$k basis and $1.177$ after pooling.

\paragraph{Pre-specified prediction ledger.}
Before training we recorded four directional predictions:
\begin{itemize}[leftmargin=2.2em,itemsep=1pt,topsep=2pt]
\item[P1.] Better optimization should shrink the matched--white FID and
  LPIPS gaps, particularly at NFE $5$--$10$, and may reverse their sign.
\item[P2.] The $\theta$-family should have an interior perceptual
  optimum, and its minimizing $\theta$ should increase under the
  converged recipe.
\item[P3.] Anti-matched should remain substantially worse on perceptual
  metrics in every tested regime.
\item[P4.] Increasing the blur to $\sigma_{\mathrm{blur}}=4.0$ should
  increase the magnitude of the matched--white effect, regardless of
  direction.
\end{itemize}
These predictions concern the direction in which the optimum moves; they
do not assume that matched must become best.

\begin{figure*}[!htb]
\centering
\includegraphics[width=0.9\textwidth]{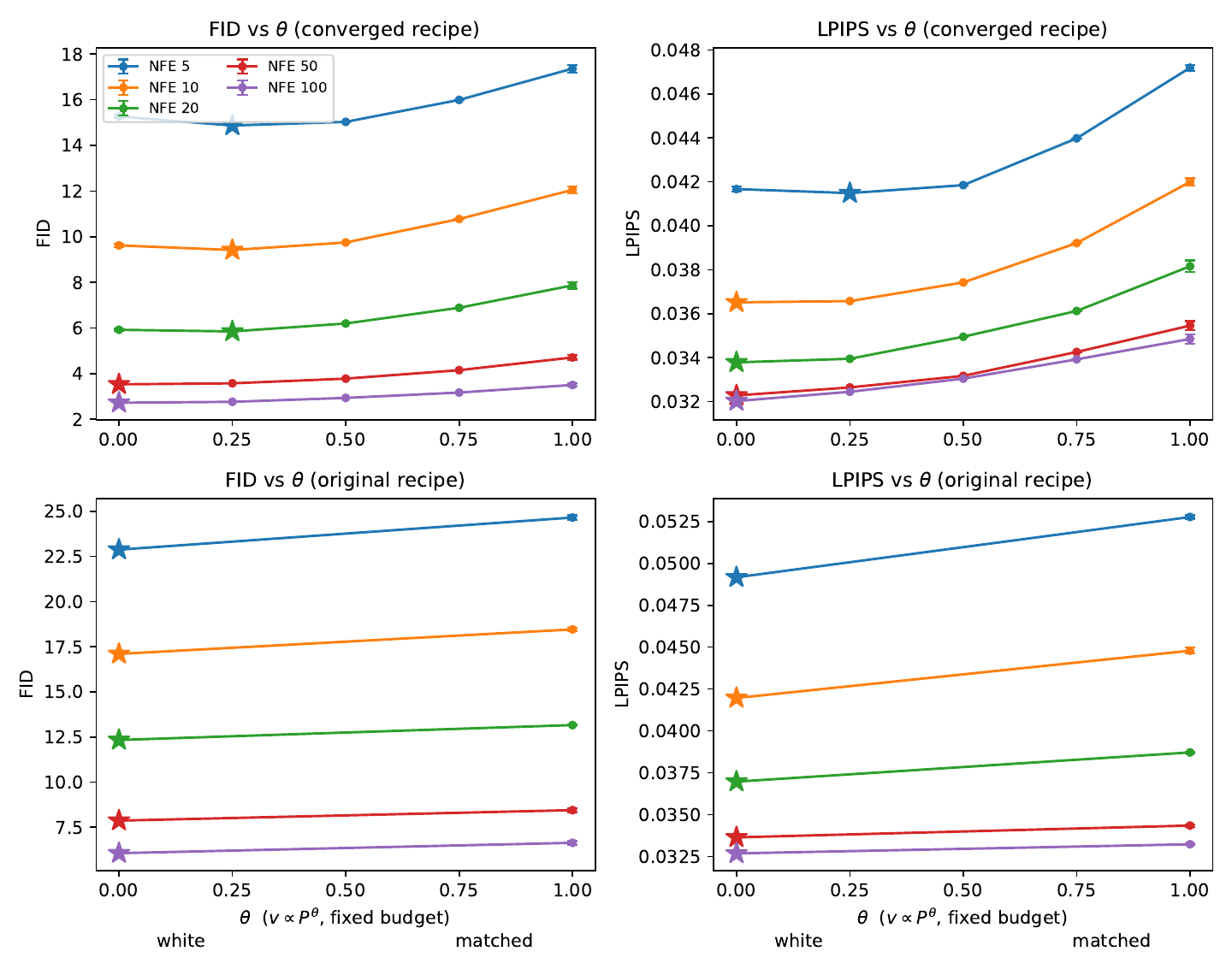}
\caption{\textbf{Reference exponent under two training recipes.} FID
(left) and LPIPS (right) against the color exponent $\theta$
($v\propto P_k^{\theta}$, fixed budget); stars mark the argmin per NFE.
\emph{Top:} converged recipe --- $\theta{=}0.25$ is best only for
NFE $\le20$, and white is best for NFE $\ge50$. \emph{Bottom:} original
recipe (endpoints only; the $\alpha$-suite is not in the
$P_k^{\theta}$ family). Better optimization widens, rather than closes,
the white--matched FID gap.}
\label{fig:theta}
\end{figure*}

\paragraph{P1 and P2.}
The converged recipe improves both endpoint references in absolute
terms: at NFE $50$, white improves from
FID $7.87$ to $4.34$ and matched from $8.45$ to $5.50$.

P1 is nevertheless refuted. The matched--white gap
\emph{increases} from $1.78$ to $2.09$ at NFE $5$, from $1.36$ to $2.43$
at NFE $10$, and from $0.58$ to $1.18$ at NFE $50$ --- at NFE $50$
roughly $27$ standard errors, given seed spreads of $\pm0.09$ (matched)
and $\pm0.02$ (white) on the pooled estimate. The LPIPS gap increases
from $0.0028$ to $0.0055$ at NFE $10$ and from $0.0007$ to $0.0032$ at
NFE $50$.

P2 receives limited support only at small NFE: FID is minimized at
$\theta=0.25$ for NFE $\le20$ (Fig.~\ref{fig:theta}), and at NFE $5$ its
value is $14.87$ against $15.27{\pm}0.07$ for white. For NFE $\ge50$ FID returns to a white
optimum, and LPIPS selects white for NFE $\ge10$. A direct cross-recipe
comparison of the minimizing $\theta$ is not identifiable, because the
original recipe did not include interior $\theta$ values; the
identifiable endpoint gap moves opposite to the prediction.

\paragraph{P3 and P4.}
P3 is supported in the original regime under both schedule grids: at
NFE $50$ anti-matched has FID $14.1$ under the uniform schedule and
$13.0$ under the $z$-optimal schedule, the worst perceptual reference in
both cases. Anti-matched was not retrained under the converged recipe,
so P3 is not evaluated there.

P4 is supported in magnitude, in the direction favoring white: at
$\sigma_{\mathrm{blur}}=4.0$ the white--matched FID gap at NFE $50$ is
$2.80$ ($6.45$ versus $9.25$; one seed per reference, pooled $50$k FID),
against $1.18$ at $\sigma_{\mathrm{blur}}=2.0$ on the same basis. The
corresponding LPIPS gaps are $0.0066$ and $0.0032$.

\paragraph{Convergence audit and scope.}
The loss fingerprint (Fig.~\ref{fig:fingerprint-conv}) shows that the
converged recipe reduces but does not eliminate low-frequency
optimization error: in the lowest radial band the median relative excess
above the Bayes floor decreases from $1.37\times$ at $150$k steps to
$1.05\times$ at $300$k for matched, and from $0.73\times$ to
$0.60\times$ for white, while mid- and high-frequency bands remain at
the floor. Under the pre-specified criterion the converged recipe is
better optimized but has not reached the exact-drift regime.

Because low-frequency fingerprint error
persists at $300$k steps, the ridge-whitening law remains untested at
true convergence. The small-NFE optimum at $\theta=0.25$ shows
that limited $P_k$-coloring can help.

\begin{figure}[!htb]
\centering
\includegraphics[width=\columnwidth]{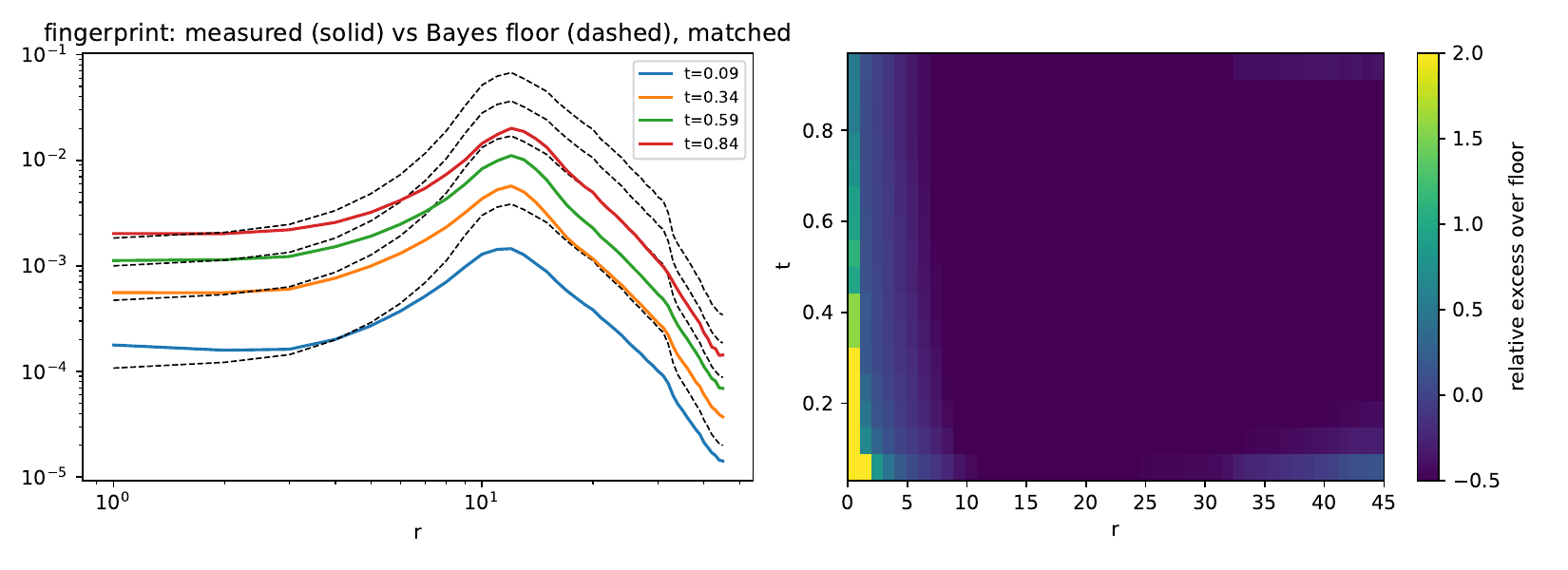}
\caption{\textbf{Convergence audit (converged recipe).} Per-(radial
band, $t$) training loss relative to the Bayes floor at $300$k steps.
The low-frequency excess shrinks relative to the original recipe but
persists; mid and high frequencies sit at the floor.}
\label{fig:fingerprint-conv}
\end{figure}

\begin{table}[!htb]
\small\centering
\caption{\textbf{FFHQ under the converged recipe} at NFE $50$ (dropout
$0$, batch size $512$, $300$k steps; all other settings as in
Table~\ref{tab:ffhq}). FID uses $50$k pooled samples. White and matched
report mean$\pm$std over five seeds; interior $\theta$ values use one
seed.}
\label{tab:ffhq-conv}
\setlength{\tabcolsep}{2.5pt}
\begin{tabular}{@{}lcccc@{}}
\toprule
reference & PSNR$\uparrow$ & SSIM$\uparrow$ & LPIPS$\downarrow$ & FID$\downarrow$ \\
\midrule
white ($\theta{=}0$)   & $\mathbf{24.81}{\pm}0.01$ & $\mathbf{0.828}$ & $\mathbf{0.0323}$ & $\mathbf{3.53}{\pm}0.02$ \\
$\theta{=}0.25$        & $24.73$ & $0.826$ & $0.0326$ & $3.58$ \\
$\theta{=}0.5$         & $24.67$ & $0.824$ & $0.0332$ & $3.78$ \\
$\theta{=}0.75$        & $24.67$ & $0.823$ & $0.0343$ & $4.15$ \\
matched ($\theta{=}1$) & $24.73{\pm}0.01$ & $0.825$ & $0.0355$ & $4.71{\pm}0.09$ \\
\bottomrule
\end{tabular}
\end{table}

\subsection{Dataset replication: CelebA}
\label{app:celeba}

\paragraph{Protocol.}
\label{app:celeba-protocol}
To test whether the white-first ordering is specific to FFHQ, we
replicate the converged recipe on CelebA at $64{\times}64$ under the
identical degradation ($\sigma_{\mathrm{blur}}=2.0$,
$\sigma_n=0.05$), with $S(k)$ and the bundle re-fit on the CelebA
training split; white, matched, $\theta{=}0.25$ and anti-matched are
trained with one seed each. Table~\ref{tab:celeba} reports FID, PSNR
and LPIPS at every evaluated NFE.

\begin{table}[!htb]
\small\centering
\caption{\textbf{CelebA replication} (converged recipe, one seed per
reference; FID uses $50$k pooled samples). Best per column in bold
within each block. White has the best FID and anti-matched the best
PSNR (SSIM tracks PSNR) at every NFE; anti-matched has the worst LPIPS
and FID at every NFE.}
\label{tab:celeba}
\setlength{\tabcolsep}{3.5pt}
\begin{tabular}{@{}lccccc@{}}
\toprule
NFE & $5$ & $10$ & $20$ & $50$ & $100$ \\
\midrule
\multicolumn{6}{@{}l}{FID$\downarrow$} \\
white           & $\mathbf{10.97}$ & $\mathbf{6.99}$ & $\mathbf{4.46}$ & $\mathbf{2.73}$ & $\mathbf{2.13}$ \\
$\theta{=}0.25$ & $11.03$ & $7.21$ & $4.75$ & $2.98$ & $2.34$ \\
matched         & $12.72$ & $8.88$ & $6.04$ & $3.85$ & $2.99$ \\
anti-matched    & $16.12$ & $11.97$ & $8.77$ & $5.88$ & $4.44$ \\
\midrule
\multicolumn{6}{@{}l}{PSNR$\uparrow$} \\
white           & $26.64$ & $26.28$ & $25.95$ & $25.60$ & $25.42$ \\
$\theta{=}0.25$ & $26.50$ & $26.17$ & $25.83$ & $25.53$ & $25.39$ \\
matched         & $26.39$ & $26.13$ & $25.85$ & $25.54$ & $25.33$ \\
anti-matched    & $\mathbf{27.27}$ & $\mathbf{27.10}$ & $\mathbf{26.89}$ & $\mathbf{26.56}$ & $\mathbf{26.28}$ \\
\midrule
\multicolumn{6}{@{}l}{LPIPS$\downarrow$} \\
white           & $0.0371$ & $0.0330$ & $\mathbf{0.0306}$ & $\mathbf{0.0295}$ & $\mathbf{0.0295}$ \\
$\theta{=}0.25$ & $\mathbf{0.0365}$ & $\mathbf{0.0329}$ & $0.0312$ & $0.0298$ & $0.0296$ \\
matched         & $0.0428$ & $0.0388$ & $0.0355$ & $0.0326$ & $0.0326$ \\
anti-matched    & $0.0510$ & $0.0463$ & $0.0416$ & $0.0362$ & $0.0335$ \\
\bottomrule
\end{tabular}
\end{table}

\paragraph{Results.}
Both FFHQ findings replicate (Table~\ref{tab:celeba}). The
distortion--perception split of Corollary~\ref{cor:dp} holds at every
NFE: anti-matched has the best PSNR/SSIM and the worst LPIPS/FID
throughout. The white-first FID ordering also holds at every NFE, and
at NFE $50$ the matched--white gap of $1.12$ is comparable to FFHQ's. On CelebA white already has the best FID
at NFE $5$, so the small-NFE interior optimum of Fig.~\ref{fig:theta}
appears dataset-dependent. 

\bibliographystyle{plain}
\bibliography{ref}

@article{shi2023diffusion,
  title={Diffusion {Schr{\"o}dinger} Bridge Matching},
  author={Shi, Yuyang and De Bortoli, Valentin and Campbell, Andrew and Doucet, Arnaud},
  journal={Advances in Neural Information Processing Systems},
  volume={36},
  pages={62183--62223},
  year={2023}
}

@article{liu20232,
  title={{I2SB}: Image-to-Image {Schr{\"o}dinger} Bridge},
  author={Liu, Guan-Horng and Vahdat, Arash and Huang, De-An and Theodorou, Evangelos A. and Nie, Weili and Anandkumar, Anima},
  journal={arXiv preprint arXiv:2302.05872},
  year={2023}
}

@article{de2021diffusion,
  title={Diffusion {Schr{\"o}dinger} Bridge with Applications to Score-Based Generative Modeling},
  author={De Bortoli, Valentin and Thornton, James and Heng, Jeremy and Doucet, Arnaud},
  journal={Advances in Neural Information Processing Systems},
  volume={34},
  pages={17695--17709},
  year={2021}
}

@article{karras2022elucidating,
  title={Elucidating the design space of diffusion-based generative models},
  author={Karras, Tero and Aittala, Miika and Aila, Timo and Laine, Samuli},
  journal={Advances in neural information processing systems},
  volume={35},
  pages={26565--26577},
  year={2022}
}

@article{kingma2021variational,
  title={Variational diffusion models},
  author={Kingma, Diederik and Salimans, Tim and Poole, Ben and Ho, Jonathan},
  journal={Advances in neural information processing systems},
  volume={34},
  pages={21696--21707},
  year={2021}
}

@article{peluchetti2023diffusion,
  title={Diffusion bridge mixture transports, Schr{\"o}dinger bridge problems and generative modeling},
  author={Peluchetti, Stefano},
  journal={Journal of Machine Learning Research},
  volume={24},
  number={374},
  pages={1--51},
  year={2023}
}

@inproceedings{bunne2023schrodinger,
  title={The schr{\"o}dinger bridge between gaussian measures has a closed form},
  author={Bunne, Charlotte and Hsieh, Ya-Ping and Cuturi, Marco and Krause, Andreas},
  booktitle={International Conference on Artificial Intelligence and Statistics},
  pages={5802--5833},
  year={2023},
  organization={PMLR}
}

@article{tang2024simplified,
  title={Simplified diffusion schr{\"o}dinger bridge},
  author={Tang, Zhicong and Hang, Tiankai and Gu, Shuyang and Chen, Dong and Guo, Baining},
  journal={arXiv preprint arXiv:2403.14623},
  year={2024}
}

@inproceedings{gushchin2024light,
  title={Light and optimal schr{\"o}dinger bridge matching},
  author={Gushchin, Nikita and Kholkin, Sergei and Burnaev, Evgeny and Korotin, Alexander},
  booktitle={Forty-first International Conference on Machine Learning},
  year={2024}
}

@inproceedings{
liu2024generalized,
title={Generalized schr{\"o}dinger Bridge Matching},
author={Guan-Horng Liu and Yaron Lipman and Maximilian Nickel and Brian Karrer and Evangelos Theodorou and Ricky T. Q. Chen},
booktitle={The Twelfth International Conference on Learning Representations},
year={2024},
url={https://openreview.net/forum?id=SoismgeX7z}
}

@article{zhang2026learning,
  title={Learning non-equilibrium diffusions with schr{\"o}dinger bridges: from exactly solvable to simulation-free},
  author={Zhang, Stephen and Stumpf, Michael},
  journal={Advances in Neural Information Processing Systems},
  volume={38},
  pages={119861--119898},
  year={2026}
}

@article{kholkin2024diffusion,
  title={Diffusion \& adversarial schr{\"o}dinger bridges via iterative proportional markovian fitting},
  author={Kholkin, Sergei and Ksenofontov, Grigoriy and Li, David and Kornilov, Nikita and Gushchin, Nikita and Suvorikova, Alexandra and Kroshnin, Alexey and Burnaev, Evgeny and Korotin, Alexander},
  journal={arXiv preprint arXiv:2410.02601},
  year={2024}
}

@article{gushchin2024adversarial,
  title={Adversarial schr{\"o}dinger bridge matching},
  author={Gushchin, Nikita and Selikhanovych, Daniil and Kholkin, Sergei and Burnaev, Evgeny and Korotin, Alexander},
  journal={Advances in Neural Information Processing Systems},
  volume={37},
  pages={89612--89651},
  year={2024}
}

@article{howard2026schrodinger,
  title={Schr{\"o}dinger Bridge Matching for Tree-Structured Costs and Entropic Wasserstein Barycentres},
  author={Howard, Samuel and Potaptchik, Peter and Deligiannidis, George},
  journal={Advances in Neural Information Processing Systems},
  volume={38},
  pages={130112--130149},
  year={2026}
}

@inproceedings{zhou2024denoising,
  title={Denoising diffusion bridge models},
  author={Zhou, Linqi and Lou, Aaron and Khanna, Samar and Ermon, Stefano},
  booktitle={International Conference on Learning Representations},
  volume={2024},
  pages={8160--8171},
  year={2024}
}

@article{wang2024implicit,
  title={Implicit image-to-image schr{\"o}dinger bridge for ct super-resolution and denoising},
  author={Wang, Yuang and Yoon, Siyeop and Jin, Pengfei and Tivnan, Matthew and Chen, Zhennong and Hu, Rui and Zhang, Li and Chen, Zhiqiang and Li, Quanzheng and Wu, Dufan},
  journal={arXiv preprint arXiv:2403.06069},
  volume={2},
  year={2024}
}

@inproceedings{zheng2025diffusion,
  title={Diffusion bridge implicit models},
  author={Zheng, Kaiwen and He, Guande and Chen, Jianfei and Bao, Fan and Zhu, Jun},
  booktitle={International Conference on Learning Representations},
  volume={2025},
  pages={81857--81884},
  year={2025}
}

@article{gushchin2025inverse,
  title={Inverse bridge matching distillation},
  author={Gushchin, Nikita and Li, David and Selikhanovych, Daniil and Burnaev, Evgeny and Baranchuk, Dmitry and Korotin, Alexander},
  journal={arXiv preprint arXiv:2502.01362},
  year={2025}
}

@article{he2024consistency,
  title={Consistency diffusion bridge models},
  author={He, Guande and Zheng, Kaiwen and Chen, Jianfei and Bao, Fan and Zhu, Jun},
  journal={Advances in Neural Information Processing Systems},
  volume={37},
  pages={23516--23548},
  year={2024}
}

@article{wang2025irbridge,
  title={Irbridge: Solving image restoration bridge with pre-trained generative diffusion models},
  author={Wang, Hanting and Jin, Tao and Lin, Wang and Wang, Shulei and Huang, Hai and Ji, Shengpeng and Zhao, Zhou},
  journal={arXiv preprint arXiv:2505.24406},
  year={2025}
}

@article{zhu2025unidb,
  title={Unidb: A unified diffusion bridge framework via stochastic optimal control},
  author={Zhu, Kaizhen and Pan, Mokai and Ma, Yuexin and Fu, Yanwei and Yu, Jingyi and Wang, Jingya and Shi, Ye},
  journal={arXiv preprint arXiv:2502.05749},
  year={2025}
}

@inproceedings{wang2026residual,
  title={Residual diffusion bridge model for image restoration},
  author={Wang, Hebaixu and Zhang, Jing and Chen, Haoyang and Guo, Haonan and Wang, Di and Ma, Jiayi and Du, Bo},
  booktitle={Proceedings of the IEEE/CVF Conference on Computer Vision and Pattern Recognition},
  pages={8375--8386},
  year={2026}
}

@article{hou2026energy,
  title={Energy-oriented Diffusion Bridge for Image Restoration with Foundational Diffusion Models},
  author={Hou, Jinhui and Zhu, Zhiyu and Hou, Junhui},
  journal={arXiv preprint arXiv:2604.10983},
  year={2026}
}

@article{yao2025regularized,
  title={Regularized schr{\"o}dinger Bridge: Alleviating Distortion and Exposure Bias in Solving Inverse Problems},
  author={Yao, Qing and Gao, Lijian and Mao, Qirong and Dong, Ming},
  journal={arXiv preprint arXiv:2511.11686},
  year={2025}
}

@inproceedings{
wyrwal2026topological,
title={Topological Flow Matching},
author={Kacper Wyrwal and Ismail Ilkan Ceylan and Alexander Tong},
booktitle={The Fourteenth International Conference on Learning Representations},
year={2026},
url={https://openreview.net/forum?id=5CM3ax45Ma}
}

@article{chertkov2026analytic,
  title={Analytic Bridge Diffusions for Controlled Path Generation},
  author={Chertkov, Michael},
  journal={arXiv preprint arXiv:2605.02961},
  year={2026}
}

@article{falck2025fourier,
  title={A fourier space perspective on diffusion models},
  author={Falck, Fabian and Pandeva, Teodora and Zahirnia, Kiarash and Lawrence, Rachel and Turner, Richard and Meeds, Edward and Zazo, Javier and Karmalkar, Sushrut},
  journal={arXiv preprint arXiv:2505.11278},
  year={2025}
}

@inproceedings{huang2024blue,
  title={Blue noise for diffusion models},
  author={Huang, Xingchang and Salaun, Corentin and Vasconcelos, Cristina and Theobalt, Christian and Oztireli, Cengiz and Singh, Gurprit},
  booktitle={ACM SIGGRAPH 2024 conference papers},
  pages={1--11},
  year={2024}
}

@article{jiralerspong2025shaping,
  title={Shaping inductive bias in diffusion models through frequency-based noise control},
  author={Jiralerspong, Thomas and Earnshaw, Berton and Hartford, Jason and Bengio, Yoshua and Scimeca, Luca},
  journal={arXiv preprint arXiv:2502.10236},
  year={2025}
}

@article{scimeca2025learning,
  title={Learning What Matters: Steering Diffusion via Spectrally Anisotropic Forward Noise},
  author={Scimeca, Luca and Jiralerspong, Thomas and Earnshaw, Berton and Hartford, Jason and Bengio, Yoshua},
  journal={arXiv preprint arXiv:2510.09660},
  year={2025}
}

@article{benita2026spectral,
  title={Spectral analysis of diffusion models with application to schedule design},
  author={Benita, Roi and Elad, Miki and Keshet, Joseph},
  journal={Advances in Neural Information Processing Systems},
  volume={38},
  pages={2073--2127},
  year={2026}
}

@inproceedings{
esteves2026spectrallyguided,
title={Spectrally-Guided Diffusion Noise Schedules},
author={Carlos Esteves and Ameesh Makadia},
booktitle={Forty-third International Conference on Machine Learning},
year={2026},
url={https://openreview.net/forum?id=5cIgeU4WOG}
}

@article{davidson2026colored,
  title={Colored Noise Diffusion Sampling},
  author={Davidson, Hadar and Issachar, Noam and Benaim, Sagie},
  journal={arXiv preprint arXiv:2605.30332},
  year={2026}
}

@article{kingma2023understanding,
  title={Understanding diffusion objectives as the elbo with simple data augmentation},
  author={Kingma, Diederik and Gao, Ruiqi},
  journal={Advances in Neural Information Processing Systems},
  volume={36},
  pages={65484--65516},
  year={2023}
}

@inproceedings{lin2024common,
  title={Common diffusion noise schedules and sample steps are flawed},
  author={Lin, Shanchuan and Liu, Bingchen and Li, Jiashi and Yang, Xiao},
  booktitle={Proceedings of the IEEE/CVF winter conference on applications of computer vision},
  pages={5404--5411},
  year={2024}
}

@article{okada2024constant,
  title={Constant Rate Scheduling: Constant-Rate Distributional Change for Efficient Training and Sampling in Diffusion Models},
  author={Okada, Shuntaro and Yoshihashi, Ryota and Kataoka, Hirokatsu and Tanaka, Tomohiro and others},
  journal={arXiv preprint arXiv:2411.12188},
  year={2024}
}

@inproceedings{blau2018perception,
  title={The perception-distortion tradeoff},
  author={Blau, Yochai and Michaeli, Tomer},
  booktitle={Proceedings of the IEEE conference on computer vision and pattern recognition},
  pages={6228--6237},
  year={2018}
}

@article{freirich2021theory,
  title={A theory of the distortion-perception tradeoff in wasserstein space},
  author={Freirich, Dror and Michaeli, Tomer and Meir, Ron},
  journal={Advances in Neural Information Processing Systems},
  volume={34},
  pages={25661--25672},
  year={2021}
}

@article{fallah2025rareflow,
  title={RareFlow: Physics-Aware Flow-Matching for Cross-Sensor Super-Resolution of Rare-Earth Features},
  author={Fallah, Forouzan and Li, Wenwen and Hsu, Chia-Yu and Lee, Hyunho and Yang, Yezhou},
  journal={arXiv preprint arXiv:2510.23816},
  year={2025}
}

@book{horn2012matrix,
  title={Matrix Analysis},
  author={Horn, R.A. and Johnson, C.R.},
  isbn={9781139788885},
  url={https://books.google.com/books?id=O7sgAwAAQBAJ},
  year={2012},
  publisher={Cambridge University Press}
}

@book{basu2006algorithms,
  title={Algorithms in real algebraic geometry},
  author={Basu, Saugata and Pollack, Richard and Roy, Marie-Fran{\c{c}}oise},
  year={2006},
  publisher={Springer}
}

@book{oksendal2010stochastic,
  title={Stochastic Differential Equations: An Introduction with Applications},
  author={{\O}ksendal, B.},
  isbn={9783642143946},
  lccn={2003052637},
  series={Universitext},
  url={https://books.google.com/books?id=EQZEAAAAQBAJ},
  year={2010},
  publisher={Springer Berlin Heidelberg}
}

@book{karatzas1991brownian,
  title={Brownian Motion and Stochastic Calculus},
  author={Karatzas, I. and Shreve, S.},
  number={v. 113},
  isbn={9780387976556},
  lccn={96167783},
  series={Brownian Motion and Stochastic Calculus},
  url={https://books.google.com/books?id=ATNy_Zg3PSsC},
  year={1991},
  publisher={Springer New York}
}

@book{rudin1976principles,
  title={Principles of Mathematical Analysis},
  author={Rudin, W.},
  isbn={9780070856134},
  lccn={75179033},
  series={International series in pure and applied mathematics},
  url={https://books.google.com/books?id=kwqzPAAACAAJ},
  year={1976},
  publisher={McGraw-Hill}
}

@book{cover2012elements,
  title={Elements of Information Theory},
  author={Cover, T.M. and Thomas, J.A.},
  isbn={9781118585771},
  lccn={2005047799},
  url={https://books.google.com/books?id=VWq5GG6ycxMC},
  year={2012},
  publisher={Wiley}
}

@article{WANG2025111627,
title = {Implicit Image-to-Image schr{\"o}dinger Bridge for image restoration},
journal = {Pattern Recognition},
volume = {165},
pages = {111627},
year = {2025},
issn = {0031-3203},
doi = {https://doi.org/10.1016/j.patcog.2025.111627},
url = {https://www.sciencedirect.com/science/article/pii/S0031320325002870},
author = {Yuang Wang and Siyeop Yoon and Pengfei Jin and Matthew Tivnan and Sifan Song and Zhennong Chen and Rui Hu and Li Zhang and Quanzheng Li and Zhiqiang Chen and Dufan Wu},
}

@inproceedings{celebA,
  title={Deep learning face attributes in the wild},
  author={Liu, Ziwei and Luo, Ping and Wang, Xiaogang and Tang, Xiaoou},
  booktitle={Proceedings of the IEEE international conference on computer vision},
  pages={3730--3738},
  year={2015}
}

@inproceedings{ffhq,
  title={A style-based generator architecture for generative adversarial networks},
  author={Karras, Tero and Laine, Samuli and Aila, Timo},
  booktitle={Proceedings of the IEEE/CVF conference on computer vision and pattern recognition},
  pages={4401--4410},
  year={2019}
}

@book{anderson2003introduction,
  title={An Introduction to Multivariate Statistical Analysis},
  author={Anderson, T.W.},
  isbn={9780471360919},
  lccn={20234317},
  series={Wiley Series in Probability and Statistics},
  url={https://books.google.com/books?id=Cmm9QgAACAAJ},
  year={2003},
  publisher={Wiley}
}

\end{document}